\documentclass{article}

\usepackage{xr-hyper}
\PassOptionsToPackage{numbers}{natbib}

\usepackage[final]{neurips_2021}
    
\usepackage[utf8]{inputenc} 
\usepackage[T1]{fontenc}    
\usepackage{hyperref}       
\usepackage{url}            
\usepackage{booktabs}       
\usepackage{amsfonts}       
\usepackage{nicefrac}       
\usepackage{microtype}      
\usepackage{xcolor}         

\usepackage{amsfonts, amsmath, amsthm, amssymb} 
\usepackage{bbm} 

\usepackage{mdframed} 
\usepackage{microtype} 
\usepackage{graphicx}
\usepackage{subfigure}
\usepackage{bm}
\usepackage{placeins} 
\usepackage{multirow}  
\usepackage{enumitem}

\newcommand{\Prob}[1]{\mathbb{P} \left\{#1\right\}}  
\newcommand{\Exp}[2][]{\mathbb{E}_{#1}\left[{#2}\right]}  
\newcommand{\ind}[1]{\mathbbm{1}\left(#1 \right)}  

\newcommand{\comment}[1]{} 

\newcommand{\DD}{\mathcal{D}}
\newcommand{\XX}{\mathcal{X}}
\newcommand{\YY}{\mathcal{Y}}

\newcommand{\R}{\mathbb{R}}  
\newcommand{\RP}{\mathbb{R}_{>0}}  

\newtheorem{theorem}{Theorem}[] 

\newtheorem{lemma}{Lemma}[]

\theoremstyle{definition}

\usepackage{multibib}
\newcites{appendix}{Appendix References}
\usepackage[title]{appendix}

\title{Combining Human Predictions with Model Probabilities via Confusion Matrices and Calibration}

\author{%
  Gavin Kerrigan$^1$    \qquad Padhraic Smyth$^1$ \qquad Mark Steyvers$^2$ \\
  $^1$Department of Computer Science \qquad $^2$Department of Cognitive Sciences \\
  University of California, Irvine \\
  \texttt{gavin.k@uci.edu} \qquad \texttt{smyth@ics.uci.edu} \qquad \texttt{mark.steyvers@uci.edu}
}

\begin{document}

\maketitle

\begin{abstract} 
An increasingly common use case for machine learning models is augmenting the abilities of human decision makers. For classification tasks where neither the human  or model are perfectly accurate, a key step in obtaining high performance is combining their individual predictions in a manner that leverages their relative strengths. In this work, we develop a set of algorithms that combine the probabilistic output of a model with the class-level output of a human. We show theoretically that the accuracy of our combination model is driven not only by the individual human and model accuracies, but also by the model's confidence.  Empirical results on image classification with CIFAR-10 and a subset of ImageNet demonstrate that such human-model combinations consistently have higher accuracies than the model or human alone, and that the parameters of the combination method can be estimated effectively with as few as ten labeled datapoints.
\end{abstract}

\section{Introduction}
\label{sect:intro}

One of the main goals of machine learning is to develop algorithms that can operate robustly in an autonomous fashion without human supervision. However, there are many applications where hybrid human-machine approaches are likely to be a preferred mode of operation, for a variety of different reasons, such as improving trust between humans and machines, and allowing for a human  or a model  to take over in situations where the one or the other lacks expertise \cite{kamar2016directions,vaughan2017making, kleinberg2018human, riedl2019human,trouille2019citizen,johnson2019no,  zahedi2021human}.  

The performance benefits of combining multiple predictors, rather than relying on a single predictor, have been clearly demonstrated in past work in a variety of fields. For example, in machine learning there is a rich vein of research over the past few decades on combining models using a variety of different estimation and algorithmic approaches \cite{kittler1998combining2,dietterich2000ensemble,kuncheva2014combining,sagi2018ensemble}. This existing line of work emphasizes that combinations of models that have diversity in how they make predictions can systematically outperform a single model.  In parallel, in the behavioral science literature, there has been extensive prior work studying combinations of human opinions where, again, diverse combinations tend to outperform any single individual \cite{hong2004groups,lamberson2012optimal}.

This naturally leads to questions about hybrid combinations of human and machine predictions, rather than just combining one type or the other. For example, one motivation for hybrid combinations is empirical evidence that human and machine classifiers do not make the same types of errors for problems such as image classification \cite{geirhos2020beyond,rosenfeld2018totally,serre2019deep}, i.e., they are diverse in their predictions. These ideas have  begun to have impact in real-world applications, where hybrid human-machine teams have been found to be effective in areas such as crowdsourcing \cite{kamar2012combining},  citizen science \cite{beck2018integrating}, speech transcription \cite{gaur2015using}, face identification \cite{phillips2018face}, and clinical radiology \cite{bien2018deep,patel2019human,rajpurkar2020chexaid}. 

In this paper we focus on a specific, simple, and important instantiation of the general problem of hybrid combinations of human and machine predictions. In our problem we consider a $K$-way classification problem such as image classification, with a single human making hard classification decisions (no confidence estimates) and a single classification model providing class-conditional probability vectors. While humans can provide confidence estimates with their label predictions, calibrated self-assessment of confidence can be difficult  \cite{keren1991calibration,kahneman1996reality,klayman1999overconfidence} as well as time-consuming.    

A key question in this context is whether the non-probabilistic information from the human predictor can be effectively combined with the probabilistic information from a machine learning model. We answer this question in the affirmative and show both theoretically and empirically how relatively simple probabilistic combination techniques can robustly outperform each of a human and machine on their own. In particular we show that a human can augment their predictions with those of a classification model, improving classification accuracy and producing calibrated predictions, even when the model is less accurate than the human. Similarly, from the model's perspective, accuracy can often be significantly improved by augmenting its' class probabilities with a human's labels, while improving calibration performance, even when the human is less accurate than the model.

\begin{figure}
    \centering
    \includegraphics[height=4.5cm]{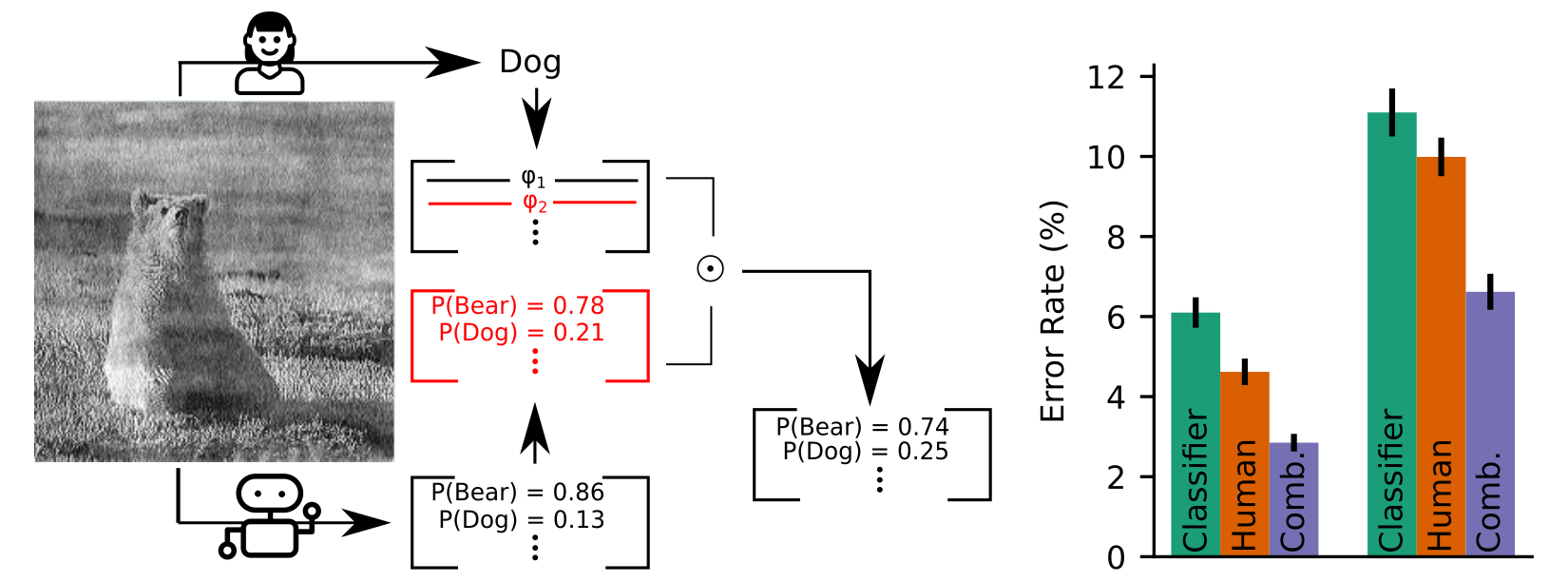}
    \caption{Left: Combining a human's label and a classifier's probabilities for an ImageNet-16H image (true label: bear). Right: Human-machine combinations (purple) achieve lower error rates on average than the human or classifier alone (ResNet-164 on CIFAR-10H, VGG-19 on ImageNet-16H).}
    \label{fig:example}
\end{figure}

Figure \ref{fig:example} shows an example using our proposed methodology for an image from the ImageNet dataset. The human incorrectly predicts the label {\tt dog}. The classification model (a VGG-19 deep network) predicts the correct label {\tt bear} with a probability of 0.86 (uncalibrated) and 0.78 (calibrated). The combined prediction is for {\tt bear} with a confidence of 0.74, with {\tt dog} having a higher confidence given the human's prediction. The histograms on the right show the overall average reduction in error rate on test images (ResNet-164 on CIFAR-10H \cite{peterson2019human}, VGG-19 on ImageNet-16H \cite{unpublishedwork}): even though the classifiers are on average less accurate (6.10\%, 11.1\% error) than the human (4.62\%, 9.99\% error), the resulting combinations have lower error rates than either (2.85\%, 6.62\% error).

The primary contributions of our work are as follows:
\begin{itemize}[nosep,leftmargin=*]
\item We propose and investigate a general framework for combining predictions using instance-level confidence from a model and class-level information from a human. The methods we propose are straightforward to implement in practice and label-efficient.\footnote{
Code for our estimation methods and experiments is available at \url{https://github.com/GavinKerrigan/conf_matrix_and_calibration}.}
\item We empirically validate our approach on the CIFAR-10H  and ImageNet-16H   image classification datasets  
and show that human-machine combinations in this context are systematically more accurate and better calibrated than either alone.
\item We develop a theoretical understanding in this framework of the key tradeoffs related to calibration and accuracy for both the individual human and model, introducing the notion of model and human confidence ratios. We illustrate how these factors affect the combination, showing for example how two models with the same accuracy but with different calibration properties can have different performance when combined with human predictions.
\end{itemize}

The paper begins in Section \ref{sect:background} by introducing notation and background concepts. 
Section \ref{sect:related} discusses  related work and in Section \ref{sect:methods} we propose a number of different estimation methods. Section \ref{sect:experiments} describes our experimental results with two image classification datasets, with individual human labelers, individual models, and combinations.  Complementing the experimental results in Section \ref{sect:theory} we develop  theoretical results that characterize combination performance. Section \ref{sect:conclusion} concludes the paper including a discussion of potential societal impact and limitations.

\section{Combining Human Labels and Model Probabilities}
\label{sect:background}

\paragraph{Notation.} We consider a $K$-ary classification problem, where the goal is to predict a label ${ y \in \YY = \{ 1, \dots, K\} }$ from features $x \in \XX$. The random variable $(x, y)$ follows an unknown distribution with support $\XX \times \YY$. We assume access to an individual human labeler represented by the function $h: \XX \to \YY$, where $h(x) \in \YY$ is the label predicted by the human. In addition, we have access to a trained machine classifier $m: \XX \to \R^K$, where $m(x)$ is the normalized probability vector output by the classifier. Both the human and classifier are assumed to be noisy labelers relative to the ground truth $y$. The true labels could be determined, for example, by expert labelers or additional information not contained in $x$.

\paragraph{Combining Predictions.}

Given an input $x \in \XX$, our goal is to predict a true label conditioned on the predictions $h(x)$ and $m(x)$. The key challenge in combining human and classifier predictions is simultaneously leveraging both the class-level outputs from the human and the predictive distributions output by the classifier. Although we focus on the particular case where $h$ is a human and $m$ is a classifier, our setup could be applied more generally to combinations of a non-probabilistic labeler (whose output is categorical) and a probabilistic labeler (whose output is a distribution over classes).

 There are a variety of functional forms that could be used to combine the predictions. We pursue a probabilistic approach, where the conditional distribution over labels that we seek can be factored via Bayes' rule as
\begin{equation}
    p(y | h(x), m(x)) \propto p( h(x) | y , m(x) ) p(y| m(x)).
\end{equation}
It is natural in this context to pursue a conditional independence (CI) approach, where the human labels $h(x)$ and the probabilistic predictions $m(x)$ are assumed to be conditionally independent given $y$. Under this assumption we can write
\begin{equation}  \label{eqn:combo_eqn}
    p(y | h(x), m(x))  \propto p( h(x) | y  ) p(y| m(x))
\end{equation}
where the right-hand side terms have a natural interpretation in terms of calibrated probabilities at the class level $\big(p( h(x) | y  )\big)$ and at the instance level $\big(p(y| m(x))\big)$.

We parameterize the term $p(h(x) | y)$ by the confusion matrix for the labeler $h$, which we denote by $\varphi$ with entries $\varphi_{ij} = p(h(x) = i | y = j)$. On the other hand, the probabilistic output of the classifier $m(x)$ may differ from $p(y | m(x))$. For example, modern neural networks tend to be overconfident in their predictions \cite{guo2017calibration}. To remedy this, post-hoc calibration maps $m(x)$ to well-calibrated probabilities via a learned calibration map with parameters $\theta$. In this work, we use $m^{\theta}(x)$ to denote the output of the classifier after applying such a calibration map. The second term in Equation \eqref{eqn:combo_eqn} is then parameterized by the calibrated classifier probabilities $m^{\theta}(x)$. Altogether, our method expresses the predicted probability for class $j$ as:
\begin{equation} \label{eqn:calibrate_confuse_combo}
\boxed{
    p(y = j | h(x) = i, m(x)) = \frac{\varphi_{ij} m_j^\theta(x)   }{\sum_{k=1}^K  \varphi_{i k} m_k^\theta(x)  }
    }
\end{equation}

The CI assumption above is common (both implicitly and explicitly) in prior work on combining predictions, such as additive classifier ensembles \cite{kuncheva2014combining, sagi2018ensemble} and (log-)linear opinion pools \cite{genest1986combining,jacobs1995methods}. As our primary motivation is to develop a relatively simple and robust methodology for combining human and model predictions, the additional functional or parametric assumptions (and parameters)  required to specify a joint model for $p(h(x), m(x) | y)$ are beyond our scope. In addition, although the CI assumption is unlikely to hold exactly, prior work \cite{kuncheva2006optimality} notes that a CI model can be an optimal discriminant even when the CI assumption is violated. As further motivation, for the two datasets we use in this paper, CIFAR-10H and ImageNet-16H, the conditional dependence of $h(x)$ and $m(x)$ appears to be relatively weak (see Appendix \ref{sect:assessing} for details).

\section{Related Work}
\label{sect:related}

We summarize below relevant aspects of related literature. While there is a significant amount of prior work in machine learning and related fields on combining predictors,  this work has in general not   addressed the specific problem of combining hard label predictions from a human with probabilistic label predictions from a model. 

\paragraph{Ensembles and Opinion Pools.} There is a rich literature in machine learning on studying predictions based on ensembles of classification models.   For non-probabilistic classifiers, the most common aggregation methods are variants of (weighted) majority voting \citep{dietterich2000ensemble, sagi2018ensemble}. However, in our case of only two predictors, a weighted majority vote ensemble can never improve accuracy over its components.  Beyond majority voting, naive Bayes aggregation \citep{xu1992methods,kuncheva2014combining} fits a class-level confusion matrix to each predictor. \citet{kim2012bayesian} develop a fully Bayesian extension of this, which relaxes the independence assumption by explicitly modeling correlations between predictors. However, because these confusion-matrix aggregation methods are at the class level they are unable to take full advantage of the instance-level uncertainties produced by the probabilistic labeler. 

In the context of aggregating predictions from multiple humans, there has been a considerable amount of prior work in the behavioral sciences and forecasting literature. Approaches include additive linear and log-linear opinion pools for subjective distributions \cite{genest1986combining,jacobs1995methods}, techniques for weighting linear combinations of real-valued human predictions \cite{lamberson2012optimal,davis2015composition}, and voting methods for combining label predictions from more than two human predictors \cite{lee2017relationship}.  A key difference between these methods and our work is that they do not address the problem of how to combine probabilistic and non-probabilistic predictions in a human-machine context.

\paragraph{Leveraging Human and Model Predictions.}  
Combining human predictions with model predictions to solve classification problems has been a topic of recent  interest in a number of different areas. For example, in  \cite{wright2017transient}  simple  averaging is used to combine the labels of multiple human annotators with the  output of a classifier for astronomical image classification, achieving better performance than with either the humans or the classifier. In crowd-sourcing, classification models have been used to automatically filter examples to improve human annotation efficiency \cite{kamar2012combining, russakovsky2015best,vaughan2017making,trouille2019citizen}.  A similar line of research  focuses on algorithmic deferral techniques where a model defers to human predictions based on the model's confidence \cite{madras2017predict, raghu2019algorithmic}, as well as work on adapting prediction models to the human decision maker \cite{mozannar2020consistent, wilder2020learning, bansal2021is,okati2021differentiable}. 
The results in \cite{mozannar2020consistent} in particular describe experiments with the   same CIFAR-10H dataset that we use in this paper. However, in addition to being different to our work in terms of its focus on deferral (rather than combining) we also note that in \cite{mozannar2020consistent} the improvements in performance are demonstrated using relatively large numbers of human labels. In contrast, as we demonstrate in Section \ref{sect:experiments} on the CIFAR-10H and ImageNet-16H datasets, our methods require only a small number of human labels to yield combined predictions that are more accurate than either human or model alone. In general, existing work on filtering and deferral strategies complements the combining methods that we develop  in this paper.  All of these approaches are  useful in a broad range of human-AI applications, but in different contexts.

 \section{Estimation Methods and Algorithms}
 \label{sect:methods}

Combining human and machine predictions via Equation \eqref{eqn:calibrate_confuse_combo} requires learning two sets of parameters: confusion matrix parameters for the human and calibration parameters for the classifier. The choice of procedure used to infer these parameters impacts the label efficiency and quality of the resulting combination. In this section, we detail several inference procedures and empirically evaluate them in the context of human-machine combinations.

To estimate our combination model, we assume access to a combination dataset ${\DD_C = \{ \left(h(x_\ell), m(x_\ell), y_\ell \right)\}_{\ell=1}^{n}}$ with human labels, machine probabilities, and ground truth labels. We assume the classifier is pre-trained with true labels on a separate training set $\DD_T$. 

\paragraph{Confusion Matrix Estimation.}
Recall that $\varphi$ denotes a a confusion matrix for the human of shape $K \times K$, where $\varphi_{ij}$ is $p(h(x) = i | y = j)$. The most straightforward estimate for this quantity is the maximum likelihood estimate, where $\varphi_{ij}$ is estimated by the number of datapoints in $\DD_C$ where the human labeler predicts $h(x) = i$ when the ground truth is $y = j$, normalized by the number of points in $\DD_C$ where $y=j$.

However, as the size of the confusion matrix is quadratic in $K$, this estimate will have high variance for small amounts of labeled data and collecting enough human labels to overcome this variance could be prohibitively expensive. We can instead take a Bayesian approach and incorporate informative prior information. Given the true label $y=j$, the human label $h | y=j \sim \text{Discrete}(\varphi_{* j} )$ is assumed to be drawn from a discrete distribution with parameters corresponding to the $j$th row in the confusion matrix. We place a conjugate Dirichlet prior $\varphi_{*j} \sim \text{Dirichlet}(\alpha_j)$ over each column with parameters $\alpha_j \in \R^k$. The prior parameters $\alpha_j$ are chosen such that
\begin{equation*}
    (\alpha_j)_i =\begin{cases} 
                  \beta & i \neq j \\
                  \gamma & i = j
                \end{cases}
\end{equation*}

That is, the prior matrix is $\gamma \in \RP$ along the diagonal and $\beta \in \RP$ on the off-diagonal. We choose $\beta$ and $\gamma$ such that the resulting Dirichlet distribution has mode equal to the train-set accuracy of the classifier, which can be obtained without additional human labels. This choice of prior reflects our belief that the confusion matrix will have a diagonally dominant structure. Posterior estimates of the confusion matrix can then be obtained straightforwardly by conjugacy.

\paragraph{Calibration Parameter Estimation.} Scaling-based calibration maps are typically fit by optimizing the  log-likelihood \cite{guo2017calibration, zhang2020mix, kull2019beyond}. In this section, we detail a Bayesian version of temperature scaling \cite{guo2017calibration}, allowing us to incorporate informative prior information. A temperature $T \in (0, 1)$ indicates underconfidence, and $T \in (1, \infty)$ indicates overconfidence. To account for this difference in scale, we place a Gaussian prior on the log-temperature $\log T = \tau \sim \mathcal{N}(\mu, \sigma^2)$. As this is a non-conjugate prior, the maximum a posteriori (MAP) temperature is estimated via gradient-based optimization. In our experiments, we choose $\sigma = 0.5$ for the CIFAR-10 models and $\sigma = 0.75$ for the ImageNet models, and we use $\mu = 0.5$ throughout. These parameters were chosen to reflect our belief that deep models tend to be overconfident and to concentrate the prior on reasonable temperature values. In Appendix \ref{sect:fully_bayesian_ts}, we derive a fully Bayesian approach where we marginalize over the posterior distribution over temperature (e.g. using Monte Carlo methods \cite{hoffman2014no}). However, we find empirically that the simpler MAP approach is more effective and, as a result, focus on MAP estimation in this  paper.

\paragraph{Learning without Ground Truth.} Requiring  both human and ground truth labels  in $\DD_C$ can be a potentially limiting assumption in domains where ground truth labels are unavailable or expensive to obtain. To avoid this, we propose an unsupervised approach that is able to learn both calibration parameters and confusion matrix parameters from a combination dataset of the form ${\DD_C = \{ \left(h(x_\ell), m(x_\ell) \right) \}}$, consisting only of human labels and machine probabilities. We treat the ground truth labels as latent and fit the required parameters using Expectation-Maximization (EM) \cite{dempster1977maximum} (details  in Appendix \ref{sect:em_details}). This approach can be seen as a novel extension of the Dawid-Skene model \cite{dawid1979maximum}, where calibration parameters are fit for the model rather than a confusion matrix. We perform both maximum likelihood and MAP estimation with this method, using the same priors as above.

\begin{table}[!t]
    \centering
    \caption{Summary of combination methods studied in this work. Except for logistic regression, parameter counts correspond to calibration using MAP temperature scaling (one parameter), and confusion matrices are fit with MAP inference. The human output is always a label.}
    \resizebox{\columnwidth}{!}{
    \begin{tabular}{lccccccc}
        \toprule
         Method Name & Acronym & Parameters & Model Output & Ground Truth? & Label Efficient? \\
         \midrule
         Logistic Regression & LR & $k^2 + 2k$ & Probabilities (P) & \checkmark & {\sf X}\\
         Calibrated Machine Probs \& Single-Parameter Confusion & SP & 2 & Probabilities (P)&  \checkmark & \checkmark\\
         \midrule
         Machine Labels \& Human Labels & L+L & $2k^2$ & Labels (L) & \checkmark & {\sf X}\\
          Calibrated Machine Probs. \& Human Labels & P+L & $k^2+1$ &Probabilities (P) & \checkmark & \checkmark \\
         Calibrated Machine Probs. \& Human Labels (EM) & P+L-EM & $k^2+1$ & Probabilities (P)& {\sf X} & \checkmark \\
        \bottomrule \\
    \end{tabular}
    }
    \label{tab:various methods}
\end{table}
 
For clarity of exposition, we focus on three types of combinations in our results:
\begin{itemize}[nosep,leftmargin=*]    
\item \textbf{Machine Labels \& Human Labels (L+L)}, a baseline where the instance-level probabilities from the model are discarded and instead a confusion matrix is fit for the model, as well as a confusion matrix for the human. The confusion matrices are estimated via supervised MAP inference. This can be viewed as a naive Bayes' combination  \cite[Chapter~4]{kuncheva2014combining} for non-probabilistic predictors.
    \item \textbf{Calibrated Machine Probabilities \& Human Labels (P+L)} combined via Equation \eqref{eqn:calibrate_confuse_combo} using \emph{supervised} MAP estimates for both the calibration parameters and confusion matrix parameters.
    \item \textbf{Calibrated Machine Probabilities \& Human Labels (P+L-EM)} combined via Equation \eqref{eqn:calibrate_confuse_combo} using \emph{unsupervised} MAP estimates fit with our EM algorithm, i.e. using human labels but no ground truth.
\end{itemize}

In terms of complexity, these methods sit between a simple model with only one parameter for the human confusion matrix \textbf{(SP)} (where the diagonal entry corresponds to the human's marginal accuracy), and a full multinomial logistic regression model (\textbf{LR}).\footnote{In fact, the CI combination with temperature scaling can be seen as a special case of multinomial logistic regression taking $m(x)$ and $h(x)$ as inputs (see Appendix \ref{sect:lr_and_ci}).} We find that \textbf{LR} can obtain slightly lower error rates in some cases, but requires significantly more labeled data than the other methods to do so. At the other extreme, \textbf{SP} is highly data efficient, but underfits compared to our preferred methods. We provide additional discussion of these methods to Appendix \ref{sect:addl_learning_curves}. The various estimation methods discussed in this section are summarized in Table \ref{tab:various methods}.

\section{Experiments}
\label{sect:experiments} 

\paragraph{Datasets and Models.}  
We evaluate various combination strategies on two pre-existing image classification datasets that include human annotations: CIFAR-10H \citep{peterson2019human} and  ImageNet-16H \citep{unpublishedwork}. CIFAR-10H contains 10-way human classifications for 10,000 images from the standard CIFAR10 test set \cite{krizhevsky2009learning}. ImageNet-16H contains 16-way human classifications for noisy images from the ImageNet test set \cite{deng2009imagenet}, distorted by phase noise at each spatial frequency based on four levels of phase noise (80, 95, 110, and 125). The human classifications for CIFAR-10H and ImageNet-16H come from the Amazon Mechanical Turk platform. 

For both CIFAR-10H and ImageNet-16H, we select a single human label for each image by randomly sampling from the available human annotations. We experiment with four CNN models on CIFAR-10H (ResNet-110, Resnet-164 \cite{he2016deep}, PreResNet-164 \cite{he2016identity}, DenseNet \cite{huang2017densely}), and eight models (four VGG-19 \cite{simonyan2014very}, four GoogLeNet \cite{szegedy2015going}) of varying accuracy on ImageNet-16H. Our models are chosen to span a range of performance, from below human accuracy to exceeding human accuracy. See Appendix \ref{sec:model_details} for details regarding our model architectures and training procedures.

Both datasets are partitioned into three disjoint subsets: (i) a model training set $\DD_T$, (ii) a combination training set $\DD_C$, and (iii) an evaluation set $\DD_E$. The model training set is the same as the suggested training split for the original CIFAR-10 and ImageNet datasets, and is used to fit the classification models. The combination training set is used to estimate any calibration parameters and confusion matrices, and the held-out evaluation set is used solely for testing. The combination training set and evaluation set are subsets of the original evaluation sets, where 70\% of the data is used for fitting the combinations and 30\% is used for evaluation.  The true labels (ground truth) for both CIFAR-10 and ImageNet correspond to the originally-provided labels for each of these datasets. In our experiments, (i) is fixed and we average over randomly selected splits for (ii) and (iii). The sampled human label for each image is fixed across all trials.

\paragraph{Calibration Methods.} In our experiments, model calibration is done by MAP temperature scaling (TS) \cite{guo2017calibration}. In Appendix \ref{sect:additional}, we experiment with two additional calibration methods: ensemble temperature scaling \cite{zhang2020mix} and I-Max binning \cite{patel2020multi}. We find that the combination performance is robust to the choice of calibration map, and hence restrict our focus to TS in this section. In addition, we find in general that combinations using uncalibrated model probabilities produce less accurate combinations than combinations using calibrated model probabilities (see Appendix \ref{sect:additional}).

\paragraph{Learning Curves.}
\begin{figure}
    \centering
    \includegraphics[]{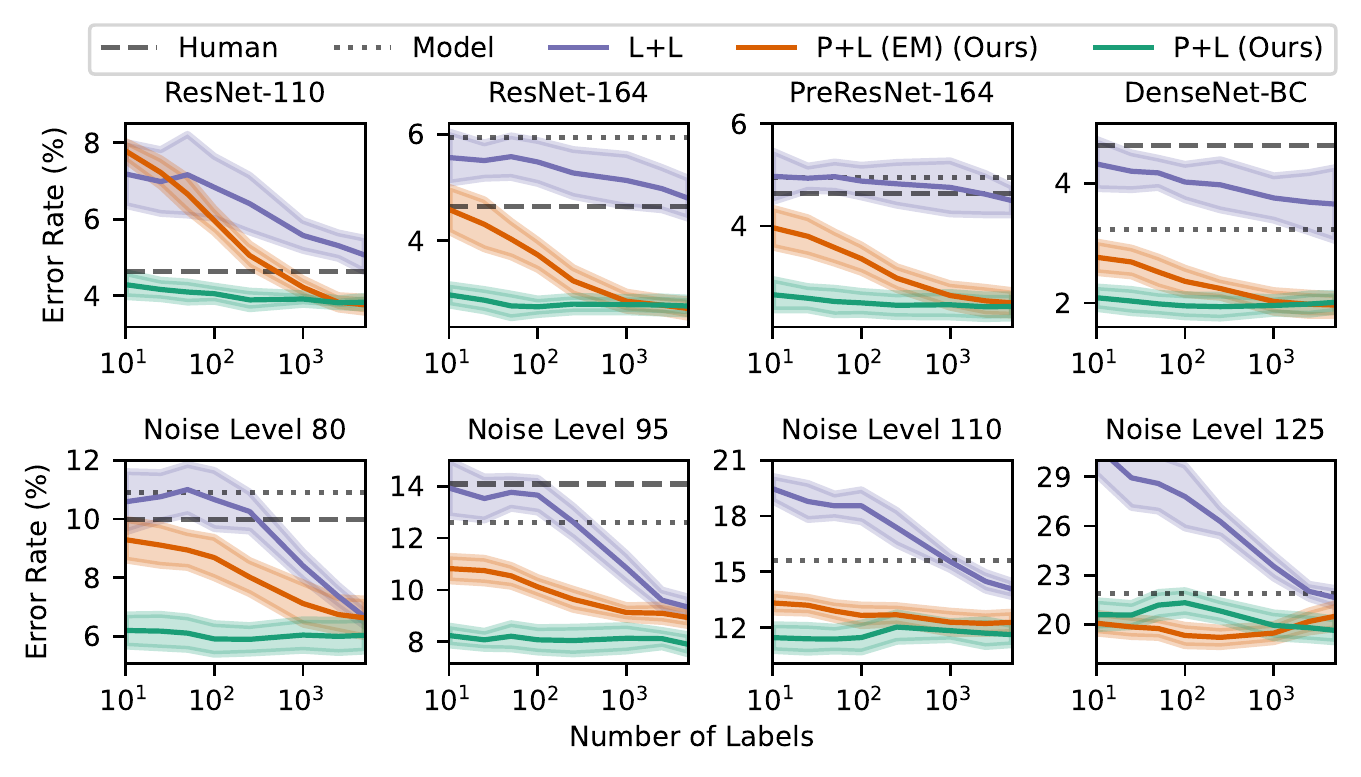}
    \caption{Learning curves for various models on CIFAR-10H (top) and VGG-19 on 
    \mbox{ImageNet-16H} (bottom). For the supervised methods (L+L, P+L), the x-axis corresponds to the number of human labels, machine predictions, and ground truth labels. For the unsupervised method (P+L (EM)), the x-axis corresponds to the number of human labels and machine predictions (but no ground truth).}
    \label{fig:learning_curves}
\end{figure}

Given that it is highly desirable to learn human-machine combinations from small amounts of data, we empirically study the data efficiency of the various inference methods previously described. In Figure \ref{fig:learning_curves}, we plot the combination error rate on the evaluation data as a function of dataset size for the four CIFAR-10 models (first row) and the VGG-19 model on ImageNet (second row). For supervised methods, the dataset size corresponds to the number of $(h(x), m(x), y)$ triples used for learning, whereas for unsupervised methods the dataset size corresponds to the number of $(h(x), m(x))$ pairs without ground truth $y$.

As Figure \ref{fig:learning_curves} demonstrates, the P+L method is able to learn a human-model combination that outperforms both the human and model alone, with very few datapoints. While the P+L-EM method requires more human labels than the P+L method, it is able to learn such a combination without any ground-truth labels. The baseline L+L method fails to learn an effective combination on the CIFAR-10 dataset, and only does so on the VGG-19 ImageNet dataset with a large number of ground truth labeled datapoints. This  demonstrates that the instance level probabilities from the model are a key component in efficiently learning human-model combinations with high accuracy.

In Appendix \ref{sect:addl_learning_curves}, we provide similar plots for GoogLeNet and for maximum likelihood based inference procedures. In general, we find that maximum likelihood estimation requires more data than MAP estimation, and hence we focus our presentation of results on MAP.

\paragraph{Calibration Properties of Combinations.}
In addition to the error rate, we study the calibration properties of P+L combinations. In Table \ref{table:cifar_calibration_mainpaper}, we report the ECE \cite{guo2017calibration}, classwise-ECE (cwECE) \cite{kull2019beyond, patel2020multi}, and negative log-likelihood (NLL) for our various CIFAR-10 models (Model) and the resulting human-machine combinations (Comb.) on the held-out evaluation set. The ECE and cwECE are evaluated using 15 bins containing an equal number of data points. In addition to P+L combinations fit with 10 or 5000 labeled datapoints, we evaluate a combination consisting of the uncalibrated classifier probabilities with the MAP human confusion matrix estimated with 5000 labeled datapoints (No Calibration).

Combining classifier probabilities with human labels generally results in a combination that is better calibrated than the model alone. Moreover, MAP TS can be fit using a very small number of labeled datapoints. Our results show that the calibration properties (of both the classifier alone and the resulting human-machine combination) significantly improve with only ten labeled examples. However, increasing the number of labeled examples to $5000$ does not result in further calibration gains. We provide similar results for our ImageNet-16H models in Appendix \ref{sect:additional_calibration}.

\begin{table}[!ht]
    \centering
    \resizebox{\columnwidth}{!}{%
    \begin{tabular}{llcccccc}
        \toprule
        \multicolumn{2}{c}{} &
        \multicolumn{2}{c}{No Calibration} & 
        \multicolumn{2}{c}{10 Datapoints} &
        \multicolumn{2}{c}{5000 Datapoints}\\
        \cmidrule(lr){3-4} \cmidrule(lr){5-6} \cmidrule(lr){7-8} 
         Metric & Model Name & Model & Comb. & Model & Comb. & Model & Comb. \\
        \midrule
        \multirow{4}{*}{ECE ($10^{-2}$)} 
        & ResNet-110 &  $5.23 \pm 0.35$ &  $2.08 \pm 0.25$ & $3.03 \pm 0.58$ & $1.30 \pm 0.23$ & $2.99 \pm 0.36$ & $1.76 \pm 0.18$ \\
        & ResNet-164 &  $2.98 \pm 0.34$ &  $1.63 \pm 0.23$ & $1.95 \pm 0.33$ & $1.25 \pm 0.18$ & $1.89 \pm 0.32$ & $1.39 \pm 0.18$ \\
        & PreResNet-164 &  $3.03 \pm 0.29$ &  $1.87 \pm 0.22$ & $2.31 \pm 0.33$ & $1.40 \pm 0.26$ & $2.27 \pm 0.31$ & $1.43 \pm 0.21$ \\
        & DenseNet-BC &  $2.18 \pm 0.27$ &  $1.53 \pm 0.20$ & $1.76 \pm 0.28$ & $1.34 \pm 0.14$ & $1.73 \pm 0.28$ & $1.27 \pm 0.13$ \\
        \midrule
        \multirow{4}{*}{cwECE ($10^{-2}$)}
        & ResNet-110 &  $0.81 \pm 0.07$ &  $0.23 \pm 0.05$ & $0.58 \pm 0.07$ & $0.24 \pm 0.05$ & $0.58 \pm 0.06$ & $0.19 \pm 0.06$ \\
        & ResNet-164 &  $0.39 \pm 0.06$ &  $0.15 \pm 0.03$ &  $0.31 \pm 0.05$ & $0.15 \pm 0.04$ & $0.31 \pm 0.05$ & $0.13 \pm 0.03$ \\
        & PreResNet-164 &  $0.29 \pm 0.04$ &  $0.13 \pm 0.03$ & $0.28 \pm 0.04$ & $0.13 \pm 0.03$ & $0.28 \pm 0.04$ & $0.13 \pm 0.03$ \\
        & DenseNet-BC &  $0.23 \pm 0.03$ &  $0.11 \pm 0.02$ &  $0.24 \pm 0.02$ & $0.12 \pm 0.02$ & $0.24 \pm 0.02$ & $0.11 \pm 0.02$ \\
        \midrule
        \multirow{4}{*}{NLL}
        &  ResNet-110 &  $0.40 \pm 0.02$ &  $0.16 \pm 0.01$ & $0.35 \pm 0.02$ & $0.15 \pm 0.01$ & $0.35 \pm 0.02$ & $0.14 \pm 0.01$ \\
        & ResNet-164 &  $0.24 \pm 0.02$ &  $0.11 \pm 0.01$ & $0.20 \pm 0.01$ & $0.10 \pm 0.01$ & $0.20 \pm 0.01$ & $0.10 \pm 0.01$ \\
        &  PreResNet-164 &  $0.23 \pm 0.02$ &  $0.13 \pm 0.02$ & $0.19 \pm 0.02$ & $0.11 \pm 0.01$ & $0.19 \pm 0.02$ & $0.10 \pm 0.01$ \\
        & DenseNet-BC &  $0.17 \pm 0.01$ &  $0.10 \pm 0.01$ & $0.14 \pm 0.01$ & $0.09 \pm 0.01$ & $0.14 \pm 0.01$ & $0.08 \pm 0.01$ \\
        \bottomrule \\
    \end{tabular}
    } 
    \caption{Calibration metrics for various CIFAR-10H models with P+L combinations. Even a small amount of labeled data (10 labels) reduces the calibration error of both the classifier and combination. For all metrics, lower is better.}
    \label{table:cifar_calibration_mainpaper}
\end{table}

\section{Theoretical Analysis}
\label{sect:theory}

\paragraph{Confidence Ratios.}
The key quantity in our analysis is the \emph{confidence ratio} of a predictor, which is a random variable representing a predictor's confidence for the correct class relative to the predictor's confidence for other classes. This quantity can be thought of as the predictor's instance-level odds for making a correct prediction. More specifically, the confidence ratios  $r_m$ and $r_h$ for machine and human labelers are defined as

\begin{equation} \label{eqn:confidence_ratios}
    r_m(x) = \frac{m^\theta_y (x) }{1 - m^\theta_y(x)} \qquad r_h(x) = \frac{ \varphi_{h(x) y}  }{1 -\varphi_{h(x) y} }
\end{equation}

We note that unlike the machine classifier, the human does not directly output such confidences -- rather, this quantity is estimated empirically through the human's confusion matrix. 

If the model has a confidence ratio of $r_m(x) > 1$ (indicating that the model has confidence greater than $0.5$ for the correct class), then the model is guaranteed to correctly label $x$. On the other hand, $r_m(x) > 1$ is not sufficient for the combination to correctly label $x$ -- instead, the model must be sufficiently confident in its prediction as well. The following theorem formalizes this notion by a lower bound on the accuracy of the combination in terms of the confidence ratios of the individual predictors. For binary classification tasks, this lower bound is achieved, i.e. we have equality.

\begin{theorem}[Combination Model Accuracy.] \label{theorem:confidence_ratios}
The accuracy of the P+L combination $c(x)$ is at least the probability that the confidence ratio for $m$ exceeds the inverse confidence ratio for $h$. 

\begin{equation}
    \Exp{\ind{c(x) = y}} \geq p(r_m(x) \geq \left( r_h(x) \right)^{-1})
\end{equation}
\end{theorem}

A detailed proof is provided in Appendix \ref{sec:proofs_1}. An analogous result holds for the combination of two probabilistic classifiers or two non-probabilistic classifiers. As our focus is on combining human predictions with model probabilities, we discuss these cases in Appendix \ref{sec:proofs_1}.

Theorem \eqref{theorem:confidence_ratios} is further illustrated in Figure \ref{fig:confidence_ratio_illustration} for a ResNet-164 classifier on CIFAR-10H (first row) and a VGG-19 classifier on ImageNet-16H (second row). For each row, we create a family of classification models where each model makes the same class-level predictions (and hence has the same error rate) but with different confidence ratio distributions. This is achieved by tempering the output probabilities of a base classifier via the map $m(X) \mapsto (m_1(X)^{1/T}, \dots, m_k(X)^{1/T}) / \sum_{i=1}^K m_i(X)^{1/T}$ with temperature $T > 0$ \cite{guo2017calibration, zhang2020mix}. In the first column of each row, the solid curve plots the error rate of the combination of a fixed human with the various classification models. Despite each classification model having the same accuracy, the accuracy of the resulting human-model combination varies.

This behavior can be explained by Theorem \eqref{theorem:confidence_ratios}, which tells us that the combination accuracy is driven not only by the human and machine accuracies but by their confidence ratios as well.  At large temperatures (purple), the classifier becomes underconfident in its predictions, and the combination error rate approaches the human error rate. At small temperatures (green), the model becomes overconfident in its predictions, and the combination error rate approaches the model error rate. When the combination is fit by our P+L method (orange), the classifier is well-calibrated (reflected by its low ECE), and the resulting combination  obtains a lower error rate than each of the under- and over-confident classifier combinations.

\begin{figure}
    \centering
    \includegraphics[height=6cm]{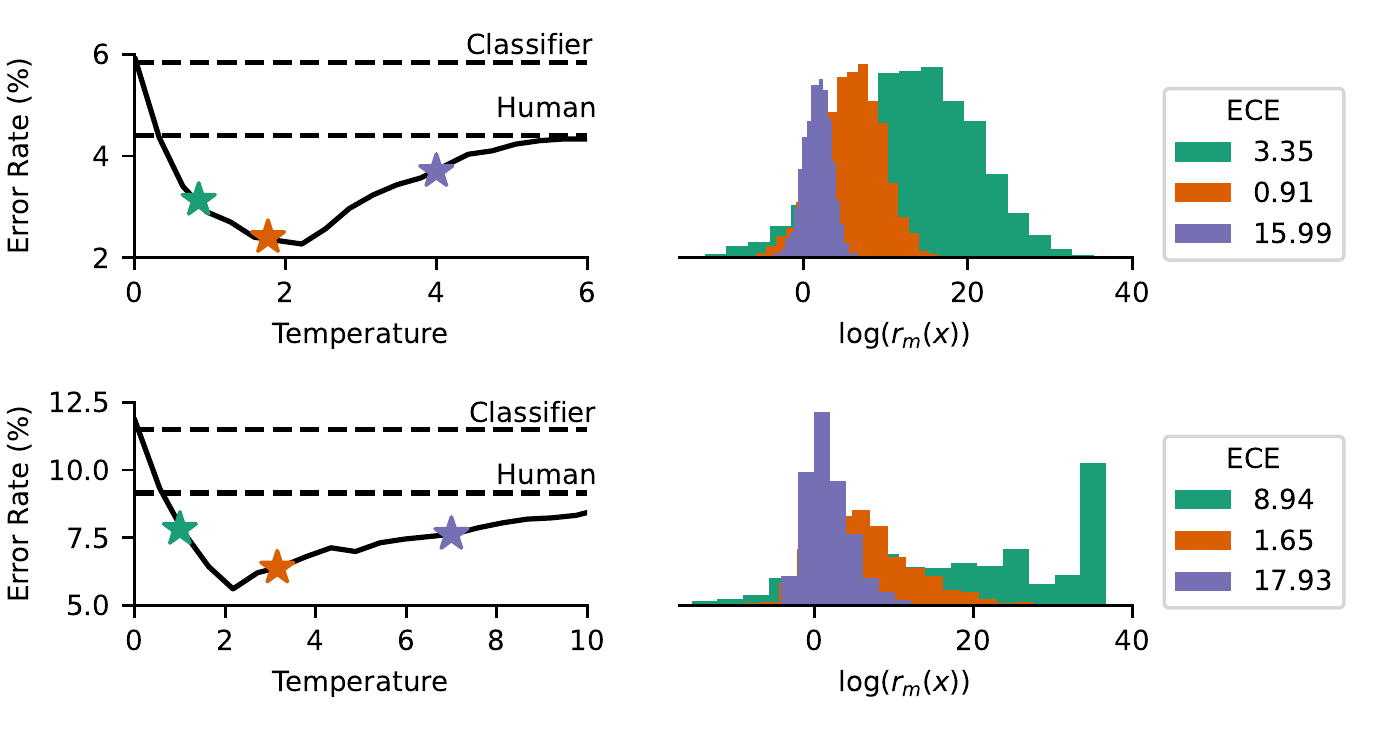}
    \caption{
    Left: Single human labeler combined with various classifiers of equal accuracy, resulting in combinations with varied accuracies. The orange point corresponds to the P+L combination. Right: First row: ResNet-164 on CIFAR-10H. Second row: VGG-19 on ImageNet-16H at noise level 80.}
    \label{fig:confidence_ratio_illustration}
\end{figure}

\paragraph{Relationship between combination and calibration/confusion error.}
We additionally quantify the estimation error in our P+L method, incurred by  empirically estimating the human confusion matrix and calibration parameters. The  result below shows we can upper bound our estimation error by the estimation error for the confusion matrix and the $\ell_1$ marginal calibration error (MCE)  \cite{kumar2019verified}. 
\begin{theorem}[Estimation Error Upper Bound.] \label{theorem:estimation_error}

Let $\eta(x, y) = |p(h(x) | y )p(y | m(x))  - m_y^\theta(x) \widehat{\varphi}_{h(x)y}|$  be the estimation error (up to normalizing constants), where $\widehat{\varphi}_{ij}$ represents an estimate of $p(h(x) = i | y = j)$. Under the CI assumption, in expectation over random $(x, y)$ pairs, 
\begin{align}
    \Exp[(X,Y)\sim\DD]{\eta(X, Y)} \leq ||\varphi - \widehat{\varphi} ||_1 + \text{MCE}(m^\theta)
\end{align}
\end{theorem}

(Proof in Appendix \ref{sec:proofs_1}). With a sufficient amount of labeled data, the confusion matrix error $||\varphi - \widehat{\varphi} ||_1$ can be made arbitrarily small (i.e. the MAP confusion matrix estimate is unbiased). This is not necessarily the case for $\text{MCE}(m^\theta)$. However, if an asymptotically unbiased calibration method is used, this result guarantees that our posterior estimation error will converge to zero.

\section{Limitations, Societal Impact, and Conclusions}
\label{sect:conclusion}

\paragraph{Limitations.} 
One limitation of our work is that our experiments only involve two datasets and both involve image classification. Thus, there is no guarantee that similar results (in terms of combined improvements) are achievable for other tasks, such as question-answering from text data or more general problem-solving tasks. Another potential limitation is our reliance on conditional independence in our approach. A reverse view of this is that both our theoretical and experimental results demonstrate that there is ample room for improving human and machine performance by combining their predictions, even without taking dependence into account.

\paragraph{Potential Societal Impacts.} Combining human and machine predictions to improve overall classification accuracy has the potential for positive societal benefit, particularly for example in high-stakes applications such as medical image diagnosis and autonomous driving. However, there are also potential negative societal impacts. For example, if there is a lack of transparency in terms of how the system operates (e.g., how predictions are being combined to arrive at a final result), augmenting an individual's predictions with  a machine's could have negative psychological consequences for the individual, such as decreasing trust, reducing individual autonomy, and eventual disengagement. 

\paragraph{Conclusions.}
We investigated methods for combining predictions using instance-level confidence from a model and class-level information from a human. Across a variety of image classification experiments our proposed combination framework  leads to systematic increases in accuracy over both the model and human alone  while requiring few human labels.   Supporting theory illustrates how combined human-model performance is affected by calibration properties of the model.

\newpage

\bibliographystyle{unsrtnat}
\bibliography{refs,refs_mark}

\begin{appendices}

\section{Proofs of Theorem \eqref{theorem:confidence_ratios} and Theorem \eqref{theorem:estimation_error}} \label{sec:proofs_1}

We provide proofs for our theoretical claims in Section \ref{sect:theory}.

\subsection{Confidence Ratios}

\begin{proof}[Proof of Theorem \eqref{theorem:confidence_ratios}]

Recall that $c(x)$ is the prediction output by Equation \eqref{eqn:calibrate_confuse_combo}. The accuracy is then bounded as follows:
\begin{align*}
    &\Exp{\ind{c(x) = y }} = \Prob{y = \arg \max_k \varphi_{h(x) k} m_k^\theta(x) } \\
    &= \Prob{\varphi_{h(x) y} m_y^\theta(x) > \max_{k \neq y } \varphi_{h(x) k} m_k^\theta(x)} \\
    &\geq \Prob{\varphi_{h(x) y} m_y^\theta(x) > \max_{k \neq y } \varphi_{h(x) k} \max_{k \neq y} m_k^\theta(x)} \\
    &\geq \Prob{\varphi_{h(x) y} m_y^\theta(x) > \left( 1 -  \varphi_{h(x) y} \right) \left( 1 -  m_y^\theta(x) \right)} \\
    &= \Prob{r_m(x) > \left(r_h(x) \right)^{-1}}
\end{align*}

In fact, we have proved the stronger but somewhat less interpretable inequality:
\begin{equation*}
    \Exp{\ind{c(x) = y }} \geq \Prob{\frac{m_y^\theta(x)}{\max_{k \neq y} m_k^\theta(x) } > \left(\frac{\varphi_{h(x) y}}{\max_{k \neq y} \varphi_{h(x)k}} \right)^{-1}}
\end{equation*}
\end{proof}

We note further that the same argument can be used to analyze the combination of two probabilistic predictors when the combination is done by pointwise multiplying their calibrated probabilities. In particular, if we have two probabilistic classifiers $m$ and $\widetilde{m}$,

\begin{equation*}
    \Exp{\ind{c(x) = y }} \geq \Prob{r_m(x) > \left(r_{\widetilde{m}} (x) \right)^{-1}}
\end{equation*}

where $r_{\widetilde{m}}$ is defined analogously to $r_m$. The proof for this statement is exactly analogous to that of Theorem \eqref{theorem:confidence_ratios}, where $\widetilde{m}_y^{\widetilde{\theta}}$ now plays the role of $\varphi_{h(x)y}$. This same argument can again be adapted for the combination of two non-probabilistic combiners, combined by parameterizing Equation \eqref{eqn:combo_eqn} with their confusion matrices.

\subsection{Estimation Error}
We begin with a useful lemma that will play a key part in our estimation error analysis. 

\begin{lemma} \label{lemma:diff_prod_leq_sum_diff}
For scalars $a_1, a_2, b_1, b_2 \in [0, 1]$, the difference of the products is at most the sum of the differences:
\begin{equation}
    |a_1 b_1 - a_2 b_2| \leq |a_1 - a_2| + |b_1 - b_2| 
\end{equation}
\end{lemma}
\begin{proof}
\begin{align*}
    |a_1 b_1 - a_2 b_2 | &= |a_1 b_1 - a_2 b_2 + a_1 b_2 - a_1 b_2 | \\
    &= |a_1 (b_1 - b_2) + b_2(a_1 - a_2)| \\
    &\leq |a_1| \cdot |b_1 - b_2| + |b_2| \cdot |a_1 - a_2| \qquad \text{(triangle inequality)} \\
    & \leq |b_1 - b_2| + |a_1 - a_2|
\end{align*}
\end{proof}

We now proceed to the proof of Theorem \eqref{theorem:estimation_error}.

\begin{proof}[Proof of Theorem \eqref{theorem:estimation_error}]

Recall that $\eta(x, y) = |p(h(x) | y )p(y | m(x))  - m_y^\theta(x) \widehat{\varphi}_{h(x)y}|$  is the estimation error for Equation \eqref{eqn:calibrate_confuse_combo} (up to normalizing constants), where $\widehat{\varphi}_{ij}$ represents an estimate of $p(h(x) = i | y = j)$. 

By the law of total expectation, we can condition on a particular value of $y$ and $h(x)$:
\begin{equation} \label{eqn:law_total_exp}
    \Exp{\eta(x, y)} = \sum_{i=1}^K \sum_{j=1}^K p(y = j) \varphi_{ij} \Exp{\eta(x, y) | y = j, h(x) = i}
\end{equation}

We now apply Lemma \eqref{lemma:diff_prod_leq_sum_diff} to the conditional expectation above:
\begin{align*}
    &\Exp{\eta(x, y) | y = j, h(x) = i} \\
    &= \Exp{| \varphi_{ij} p(y = j | m(x)) - \widehat{\varphi}_{ij} m_j^\theta(X) |  \bigg| y = j, h(x) = i } \\
    &\leq \Exp{|\varphi_{ij} - \widehat{\varphi}_{ij} | + |p(y = j | m(x)) - m_j^\theta(x) | \bigg| y = j, h(x) = i  } \\
    &= |\varphi_{ij} - \widehat{\varphi}_{ij} | + \Exp{|p(y = j | m(x)) - m_j^\theta(x) | \bigg| y = j } 
\end{align*}

We additionally employ the conditional independence assumption to arrive at the last line. 

Plugging this back in to Equation \eqref{eqn:law_total_exp}, we obtain
\begin{align*}
    \Exp{\eta(x, y)} \leq &\sum_{i=1}^K \sum_{j=1}^K P(y = j) \varphi_{ij} |\varphi_{ij} - \widehat{\varphi}_{ij} | \\
    + &\sum_{j=1}^K p(y = j) \Exp{|p(y = j | m(x)) - m_j^\theta(x) | \bigg| y = j } 
\end{align*}

Since $\varphi_{ij}, p(y = 1) \leq 1$, the first summand is at most $\sum_{i=1}^K \sum_{j=1}^K |\varphi_{ij} - \widehat{\varphi}_{ij} | = ||\varphi - \widehat{\varphi}||_1$. In fact, the first summand is typically much smaller than $||\varphi - \widehat{\varphi}||_1$ -- for example, if all classes are equally likely, the first summand is at most $\frac{1}{K} ||\varphi - \widehat{\varphi} ||_1 $.

The second summand is readily recognized as the $\ell_1$ marginal calibration error \cite{kumar2019verified}. 
\end{proof}

\newpage
\section{Assessing Conditional Independence/Dependence in CIFAR-10H and Imagenet-16H Datasets}
\label{sect:assessing}

We investigate the degree to which our conditional independence assumption  is satisfied empirically in the datasets used in the paper.  Specifically, of interest is the assumption of conditional independence of $m(x)$ and $h(x)$, given $y$. Assessing conditional independence is not straightforward given that $m(x)$ is a $K$-dimensional real-valued vector and $h(x)$ and $y$ each take one of $K$ categorical values, with $K=10$ for CIFAR-10H and $K=16$ for ImageNet-16H. While there exist statistical tests for assessing conditional independence for categorical random variables, with real-valued variables the situation is less straightforward and there are multiple options such as different non-parametric tests involving different tradeoffs \citeappendix{runge2018conditional,marx2019testing, mukherjee2020ccmi,berrett2020conditional}.

Given these issues we investigate the degree of conditional dependence using two relatively simple approaches. The first approach looks at the conditional mutual information (CMI) between the predicted label from the model and the predicted label from the human, conditioned on the true label. While this is indirect, in that it does not use the real-valued scores, it does allow us to measure CMI in a straightforward manner given that all the variables involved are categorical. The CMI is defined as
\[
\mbox{CMI}(M ; H | Y) \ = \ \sum_y p(y) \sum_{m, h} p(m, h | y) \log \frac{ p(m, h| y)}{p(m|y) p(h|y)}
\] 
where $M, H, Y$ are the $K$-ary random variables for the model, human, and true labels respectively (taking values $m, h, y$). The inner sum over $m, h$ is the mutual information between $M$ and $H$ conditioned on a particular value of $Y=y$. All probabilities were estimated using relative frequencies (maximum likelihood) from the evaluation sets for each dataset.

\begin{table}[!t]
    \centering
    \caption{Conditional and unconditional mutual information for various datasets and models.}
    \begin{tabular}{ccccc}
        \toprule
         Dataset & Model & Noise & CMI$(M ; H | Y)$ & MI$(M ; H)$ \\
         \midrule
          CIFAR-10H &  Densenet & &  0.030 & 2.829 \\
          CIFAR-10H &  PreResnet-164 & &  0.043 & 2.770 \\
          CIFAR-10H &  Resnet-110 & &  0.037 & 2.404 \\
          CIFAR-10H &  Resnet-164 & &  0.038 & 2.707 \\
          ImageNet-16H &  VGG-19 & 80 &  0.119 & 2.954 \\
          ImageNet-16H &  VGG-19 & 95 &  0.174 & 2.816 \\
          ImageNet-16H &  VGG-19 & 110 &  0.230 & 2.277 \\
          ImageNet-16H &  VGG-19 & 125 &  0.314 & 1.527 \\
          ImageNet-16H &  GoogLeNet & 80 &  0.121 & 2.825 \\
          ImageNet-16H &  GoogLeNet & 90 &  0.161 & 2.643 \\
          ImageNet-16H &  GoogLeNet & 110 &  0.260 & 2.182 \\
          ImageNet-16H &  GoogLeNet & 125 &  0.364 & 1.421 \\
        \bottomrule \\
    \end{tabular}
    \label{tab:cmi}
\end{table} 

Table \ref{tab:cmi} shows the results for the 4 different models for CIFAR-10H and the $2 \times 4$ different combinations of models and noise for ImageNet-16H. To put the CMI numbers on an interpretable scale, we also compute the (unconditional) mutual information between $M$ and $H$ in each case. If $M$ and $H$ are truly independent conditioned on $Y$, then the true CMI values should be 0.

The broad conclusion from Table \ref{tab:cmi} is that  for the CIFAR-10H there appears to be little to no conditional dependence (of model labels and human labels, given true labels) given that the CMI values are very close to 0. For the ImageNet-16H data the CMI values are higher, suggesting evidence for weak conditional dependence in this dataset, particularly at high noise levels where neither the human or the model are very accurate. 

Figures \ref{fig:cifarCI}, \ref{fig:vggCI}, \ref{fig:googlenetCI} show  the results of another assessment, now using model probabilities, for the 4 models for the CIFAR-10H data, for VGG-19 on ImageNet-16H, and for GoogLeNet on ImageNet-16H, respectively. The x-axis in each plot is the mean probability from the model for the true label $y$, conditioned on $Y=y$. The $y$-axis shows the mean probability (in red) from the model for the true label $y$, conditioned now on {\bf both}  $Y=y$ {\bf and} $H=y$, i.e., conditioned on the event that the human also predicts the true label. 

If the model's probabilities for the true labels are independent of $H=y$, then the x and y values should be the same (i.e., on the diagonal). The degree to which these points (in red) are not on the diagonal is an indication of some conditional dependence of the model's probabilities on the human labels $h$. The red points are generally close to the diagonal, or slightly above (indicating, not surprisingly, that if the human predicts the true label, the model's probability for the true label tends to increase slightly (if at all) rather than decrease.) To put these values on an appropriate scale we also compute (empirically from the data) the maximum possible increase that could occur, when additionally conditioning on the human label $h$ being correct (the black points). The conclusions are   similar to what we found with conditional mutual information, namely, that there is little indication of conditional dependence in the CIFAR-10H data, and some indication of dependence in the ImageNet-16H data, particularly for higher noise levels.

\begin{figure}[hbtp]
    \centering
    \includegraphics[width=4in]{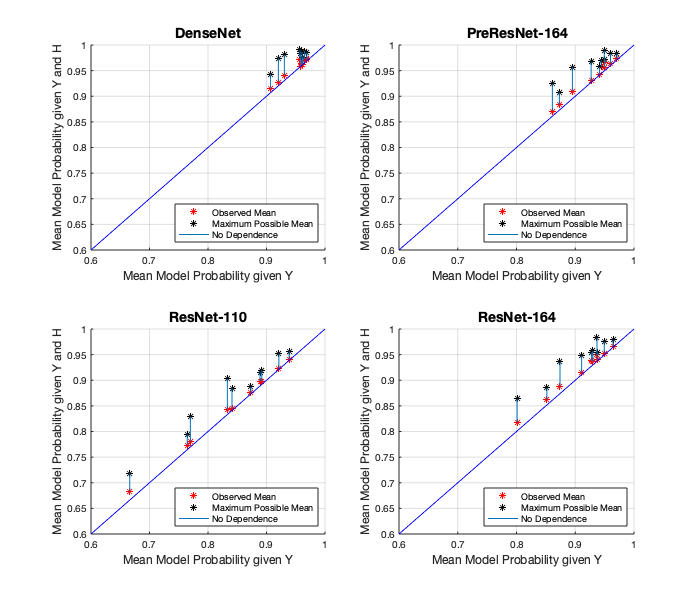} 
    \caption{Change in expected values of model probabilities on CIFAR-10H data for the true class $y$, conditioning on just $y$ (x-axis), versus conditioning on both $y$ and $h(x) = y$ (y-axis, in red).}
    \label{fig:cifarCI}
\end{figure}

\begin{figure}[hbtp]
    \centering
    \includegraphics[width=4in]{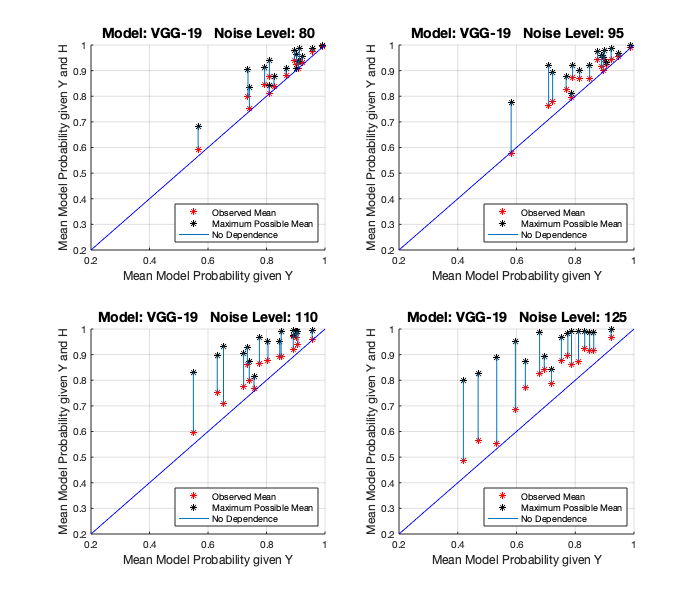} 
    \caption{Same as Figure \ref{fig:cifarCI} but for VGG-19 models on ImageNet-16H data.}
    \label{fig:vggCI}
\end{figure}

\begin{figure}[hbtp]
    \centering
    \includegraphics[width=4in]{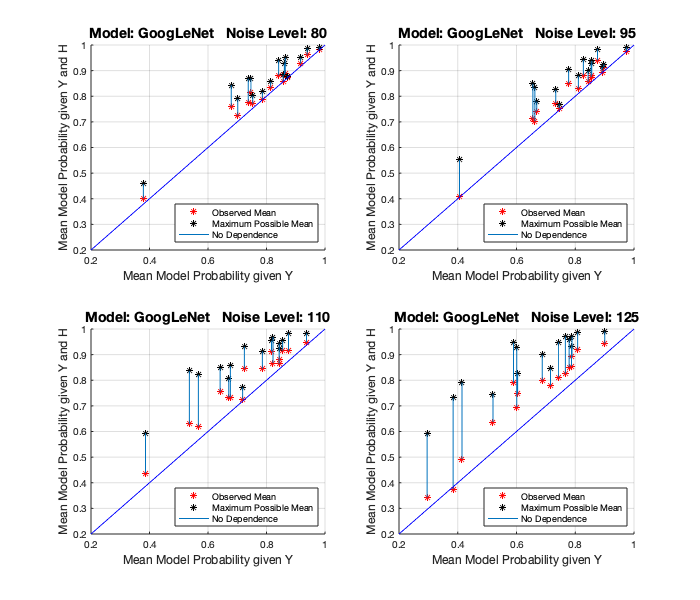} 
    \caption{Same as Figure \ref{fig:cifarCI} but for   GoogLeNet models on ImageNet-16H data.}
    \label{fig:googlenetCI}
\end{figure}

\FloatBarrier
\newpage
\section{Calibration Methods and Uncalibrated Combinations}
\label{sect:additional}

In this section we provide additional empirical results on CIFAR-10H and ImageNet-16H. In particular, we evaluate several different calibration methods (MAP TS (as used in the main paper) \cite{guo2017calibration}, Ensemble TS \cite{zhang2020mix}, IMax Binning \cite{patel2020multi}). We also compare to the L+L combination, and the P+L combination of the uncalibrated model probabilities with the human labels (Uncalibrated). The error rate of the human alone (Human) and model alone (Model) are provided for context.

In most cases, human-machine combinations using calibrated probabilities outperform those using uncalibrated probabilities. Moreover, in some cases we obtain small gains in performance by using a more complex calibration map (IMax Binning), but it is not clear how to incorporate prior information with this method. As prior information is useful in increasing the label efficiency and decreasing the error rate of the combination, our focus in the main paper is on MAP TS as our calibration method.

All tables in this section correspond to error rates ($\pm$ one standard deviation) averaged across 25 different random seeds. The combinations (P+L, Equation \eqref{eqn:calibrate_confuse_combo}) are fit using 5000 labeled data points on CIFAR-10H and using between 5067 and 5152 data points on ImageNet-16H (varies by noise level). The combinations are evaluated using 3000 data points on CIFAR-10H and using between 2171 and 2208 on ImageNet-16H.

\begin{table}[!h]
    \centering
    \resizebox{\columnwidth}{!}{%
    \begin{tabular}{lccccccc}
    \toprule
    \multicolumn{3}{c}{} &
    \multicolumn{5}{c}{Combination} \\
    \cmidrule(lr){4-8}
    Model Name & Human & Model & L+L & Uncalibrated & TS &  ETS & IMax \\
    \midrule 
    ResNet-110 &     $4.62 \pm 0.33$ &  $11.28 \pm 0.44$ &  $4.70 \pm 0.36$ &  $4.40 \pm 0.25$ &  $3.83 \pm 0.15$ &  $3.76 \pm 0.25$ &  $\bm{3.80 \pm 0.24}$ \\
    ResNet-164     & --- &   $6.10 \pm 0.38$ &  $4.71 \pm 0.37$ &  $3.05 \pm 0.23$ &  $\bm{2.78 \pm 0.15}$ &  $2.82 \pm 0.23$ &  $2.85 \pm 0.23$ \\
    PreResNet-164  & --- &   $5.00 \pm 0.36$ &  $4.36 \pm 0.39$ &  $2.90 \pm 0.22$ &  $\bm{2.43 \pm 0.22}$ &  $2.46 \pm 0.25$ &  $\bm{2.43 \pm 0.26}$ \\
    DenseNet-BC    &  --- &   $3.25 \pm 0.30$ &  $3.39 \pm 0.32$ &  $2.22 \pm 0.21$ &  $\bm{2.01 \pm 0.15}$ &  $2.17 \pm 0.17$ &  $2.04 \pm 0.18$ \\
    \bottomrule \\
    \end{tabular}
    }
    \caption{Error rates (\%, $\pm$ one standard deviation) averaged over 25 seeds on CIFAR-10H for various classifiers.}
    \label{table:cifar10_error_rates}
\end{table}

\begin{table}[!h]
    \centering
    \resizebox{\columnwidth}{!}{%
    \begin{tabular}{lccccccc}
    \toprule
    \multicolumn{2}{c}{} &
    \multicolumn{5}{c}{Combination} \\
    \cmidrule(lr){3-7}
    Human & Model & L+L & Uncalibrated & TS &  ETS & IMax \\
    \midrule 
    $9.99 \pm 0.48$ &  $11.10 \pm 0.60$ &   $6.78 \pm 0.42$ &   $7.52 \pm 0.52$ &    $\bm{6.03 \pm 0.54}$ &   $6.79 \pm 0.44$ &   $6.31 \pm 0.46$ \\
    $14.07 \pm 0.70$ &  $12.58 \pm 0.53$ &   $9.01 \pm 0.57$ &   $9.02 \pm 0.44$ &   $\bm{7.89 \pm 0.37}$ &   $9.32 \pm 0.49$ &   $8.62 \pm 0.47$ \\
     $22.99 \pm 0.71$ &  $15.51 \pm 0.62$ &  $14.07 \pm 0.82$ &  $12.59 \pm 0.53$ &  $\bm{11.62 \pm 0.54}$ &  $13.18 \pm 0.59$ &  $12.30 \pm 0.57$ \\
    $39.76 \pm 0.75$ &  $22.07 \pm 0.69$ &  $21.89 \pm 0.66$ &  $\bm{19.45 \pm 0.62}$ &  $19.63 \pm 0.70$ &  $20.74 \pm 0.62$ &  $20.47 \pm 0.63$ \\
    \bottomrule \\
    \end{tabular}
    }
    \caption{Error rates (\%, $\pm$ one standard deviation) averaged over 25 seeds, VGG-19 on ImageNet-16H. Each row corresponds to a different noise level (80, 95, 110, 125).}
    \label{table:imagenet_error_rates_vgg19}
\end{table}

\begin{table}[!h]
    \centering
    \resizebox{\columnwidth}{!}{%
    \begin{tabular}{lccccccc}
    \toprule
    \multicolumn{2}{c}{} &
    \multicolumn{5}{c}{Combination} \\
    \cmidrule(lr){3-7}
    Human & Model & L+L & Uncalibrated & TS &  ETS & IMax \\
    \midrule 
   $9.99 \pm 0.48$ &  $14.48 \pm 0.70$ &   $7.33 \pm 0.39$ &   $8.06 \pm 0.50$ &   $\bm{6.80 \pm 0.47}$ &   $7.73 \pm 0.41$ &   $7.66 \pm 0.44$ \\
    $14.07 \pm 0.70$ &  $17.22 \pm 0.72$ &   $9.93 \pm 0.73$ &  $10.66 \pm 0.52$ &  $\bm{9.67 \pm 0.40}$ &  $10.23 \pm 0.51$ &  $10.05 \pm 0.49$ \\
     $22.99 \pm 0.71$ &  $19.09 \pm 0.75$ &  $15.43 \pm 0.71$ &  $14.78 \pm 0.64$ &   $\bm{14.09 \pm 0.55}$ &  $14.76 \pm 0.53$ &  $14.53 \pm 0.53$ \\
    $39.76 \pm 0.75$ &  $27.06 \pm 0.47$ &  $25.64 \pm 0.43$ &  $23.06 \pm 0.72$ &  $\bm{22.60 \pm 0.63}$ &  $24.38 \pm 0.64$ &  $23.91 \pm 0.69$ \\
    \bottomrule \\
    \end{tabular}
    }
    \caption{Error rates (\%, $\pm$ one standard deviation) averaged over 25 seeds, GoogLeNet on ImageNet-16H. Each row corresponds to a different noise level (80, 95, 110, 125).}
    \label{table:imagenet_error_rates_googlenet}
\end{table}

\FloatBarrier
\newpage
\section{Calibration Properties of Combinations}
\label{sect:additional_calibration}

We further study the calibration properties of human-machine (P+L) combinations. The results in this Appendix are analogous to the results in Table \ref{table:cifar_calibration_mainpaper} for our ImageNet-16H models, where we show various calibration metrics as we vary the number of labeled datapoints used for fitting the combination. In general, we find that using only a small number of labeled datapoints (10 in our experiments) is sufficient, and we do not observe further improvements in calibration by using more labeled data (5000 points in our experiments) to fit the combination.

In addition, we investigate whether the resulting human-machine combination can be further calibrated. We calibrate the resulting human-machine combinations (with MAP TS) using the same data used to fit the combination, i.e. 5000 labeled datapoints (Recal. Comb.). We find that it is possible to further reduce the ECE of the combinations, but other metrics only see small improvements. However, we note that this does not affect the error rate of the combination, as MAP TS is accuracy-preserving.

\begin{table}[!ht]
    \centering
    \resizebox{\columnwidth}{!}{%
    \begin{tabular}{llccccccc}
        \toprule
        \multicolumn{2}{c}{} &
        \multicolumn{2}{c}{No Calibration} & 
        \multicolumn{2}{c}{10 Datapoints} &
        \multicolumn{3}{c}{5000 Datapoints}\\
        \cmidrule(lr){3-4} \cmidrule(lr){5-6} \cmidrule(lr){7-9} 
         Metric & Model Name & Model & Comb. & Model & Comb. & Model & Comb. & Recal. Comb. \\
        \midrule
        \multirow{4}{*}{ECE ($10^{-2}$)} 
        & ResNet-110 &  $5.23 \pm 0.35$ &  $2.08 \pm 0.25$ & $3.03 \pm 0.58$ & $1.30 \pm 0.23$ & $2.99 \pm 0.36$ & $1.76 \pm 0.18$ & $0.85 \pm 0.22$ \\
        & ResNet-164 &  $2.98 \pm 0.34$ &  $1.63 \pm 0.23$ & $1.95 \pm 0.33$ & $1.25 \pm 0.18$ & $1.89 \pm 0.32$ & $1.39 \pm 0.18$ & $0.84 \pm 0.20$ \\
        & PreResNet-164 &  $3.03 \pm 0.29$ &  $1.87 \pm 0.22$ & $2.31 \pm 0.33$ & $1.40 \pm 0.26$ & $2.27 \pm 0.31$ & $1.43 \pm 0.21$ & $1.06 \pm 0.21$ \\
        & DenseNet-BC &  $2.18 \pm 0.27$ &  $1.53 \pm 0.20$ & $1.76 \pm 0.28$ & $1.34 \pm 0.14$ & $1.73 \pm 0.28$ & $1.27 \pm 0.13$ & $0.95 \pm 0.18$ \\
        \midrule
        \multirow{4}{*}{cwECE ($10^{-2}$)}
        & ResNet-110 &  $0.81 \pm 0.07$ &  $0.23 \pm 0.05$ & $0.58 \pm 0.07$ & $0.24 \pm 0.05$ & $0.58 \pm 0.06$ & $0.19 \pm 0.06$ & $0.19 \pm 0.04$ \\
        & ResNet-164 &  $0.39 \pm 0.06$ &  $0.15 \pm 0.03$ &  $0.31 \pm 0.05$ & $0.15 \pm 0.04$ & $0.31 \pm 0.05$ & $0.13 \pm 0.03$ & $0.14 \pm 0.03$ \\
        & PreResNet-164 &  $0.29 \pm 0.04$ &  $0.13 \pm 0.03$ & $0.28 \pm 0.04$ & $0.13 \pm 0.03$ & $0.28 \pm 0.04$ & $0.13 \pm 0.03$ & $0.13 \pm 0.03$ \\
        & DenseNet-BC &  $0.23 \pm 0.03$ &  $0.11 \pm 0.02$ &  $0.24 \pm 0.02$ & $0.12 \pm 0.02$ & $0.24 \pm 0.02$ & $0.11 \pm 0.02$ & $0.10 \pm 0.02$ \\
        \midrule
        \multirow{4}{*}{NLL}
        &  ResNet-110 &  $0.40 \pm 0.02$ &  $0.16 \pm 0.01$ & $0.35 \pm 0.02$ & $0.15 \pm 0.01$ & $0.35 \pm 0.02$ & $0.14 \pm 0.01$ & $0.12 \pm 0.01$ \\
        & ResNet-164 &  $0.24 \pm 0.02$ &  $0.11 \pm 0.01$ & $0.20 \pm 0.01$ & $0.10 \pm 0.01$ & $0.20 \pm 0.01$ & $0.10 \pm 0.01$ & $0.09 \pm 0.01$ \\
        &  PreResNet-164 &  $0.23 \pm 0.02$ &  $0.13 \pm 0.02$ & $0.19 \pm 0.02$ & $0.11 \pm 0.01$ & $0.19 \pm 0.02$ & $0.10 \pm 0.01$&$0.08 \pm 0.01$ \\
        & DenseNet-BC &  $0.17 \pm 0.01$ &  $0.10 \pm 0.01$ & $0.14 \pm 0.01$ & $0.09 \pm 0.01$ & $0.14 \pm 0.01$ & $0.08 \pm 0.01$ &$0.07 \pm 0.01$ \\
        \bottomrule \\
    \end{tabular}
    } 
    \caption{Calibration metrics on CIFAR-10H.}
    \label{table:cifar_calibration_appendix}
\end{table}

\begin{table}[!ht]
    \centering
    \resizebox{\columnwidth}{!}{%
    \begin{tabular}{llccccccc}
        \toprule
        \multicolumn{2}{c}{} &
        \multicolumn{2}{c}{No Calibration} & 
        \multicolumn{2}{c}{10 Datapoints} &
        \multicolumn{3}{c}{5000 Datapoints}\\
        \cmidrule(lr){3-4} \cmidrule(lr){5-6} \cmidrule(lr){7-9} 
         Metric & Noise Level & Model & Comb. & Model & Comb. & Model & Comb. & Recal. Comb \\
        \midrule
        \multirow{4}{*}{ECE ($10^{-2}$)} 
        & 80 & $8.54 \pm 0.54$ & $5.17 \pm 0.49$ & $7.30 \pm 0.69$ & $3.91 \pm 0.53$ &$7.15 \pm 0.60$ & $4.01 \pm 0.42$ &$3.17 \pm 0.42$ \\
        & 95 & $8.96 \pm 0.48$ & $5.72 \pm 0.39$ & $7.49 \pm 0.77$ & $4.93 \pm 0.36$ &$7.26 \pm 0.51$ & $4.56 \pm 0.37$ &$3.23 \pm 0.35$ \\
        & 110 & $9.76 \pm 0.53$ & $7.81 \pm 0.48$ & $7.81 \pm 0.91$ & $6.31 \pm 0.67$ &$7.24 \pm 0.56$ & $6.07 \pm 0.53$ &$3.92 \pm 0.49$ \\
        & 125 & $11.81 \pm 0.64$ & $10.89 \pm 0.52$ & $7.34 \pm 1.45$ & $10.21 \pm 0.64$ &$7.29 \pm 0.56$ & $8.46 \pm 0.68$ &$4.49 \pm 0.60$ \\
        \midrule
        \multirow{4}{*}{cwECE ($10^{-2}$)}
        & 80 & $1.10 \pm 0.07$ & $0.68 \pm 0.05$ & $1.01 \pm 0.07$ & $0.59 \pm 0.06$ &$1.01 \pm 0.06$ & $0.54 \pm 0.05$ &$0.56 \pm 0.04$ \\
        & 95 & $1.18 \pm 0.06$ & $0.82 \pm 0.05$ & $1.13 \pm 0.06$ & $0.73 \pm 0.06$ &$1.12 \pm 0.06$ & $0.72 \pm 0.04$ &$0.69 \pm 0.04$ \\
        & 110 & $1.44 \pm 0.06$ & $1.14 \pm 0.06$ & $1.38 \pm 0.07$ & $1.04 \pm 0.08$ &$1.36 \pm 0.07$ & $1.03 \pm 0.07$ &$0.96 \pm 0.05$ \\
        & 125 & $1.98 \pm 0.06$ & $1.73 \pm 0.07$ & $1.86 \pm 0.05$ & $1.54 \pm 0.07$ &$1.85 \pm 0.04$ & $1.52 \pm 0.06$ &$1.45 \pm 0.06$ \\
        \midrule
        \multirow{4}{*}{NLL}
        & 80 & $0.71 \pm 0.05$ & $0.49 \pm 0.04$ & $0.53 \pm 0.05$ & $0.37 \pm 0.04$ &$0.52 \pm 0.03$ & $0.34 \pm 0.03$ &$0.27 \pm 0.02$ \\
        & 95 & $0.70 \pm 0.03$ & $0.52 \pm 0.03$ & $0.55 \pm 0.04$ & $0.41 \pm 0.03$ &$0.54 \pm 0.03$ & $0.39 \pm 0.02$ &$0.32 \pm 0.02$ \\
        & 110 & $0.73 \pm 0.03$ & $0.61 \pm 0.04$ & $0.60 \pm 0.04$ & $0.51 \pm 0.05$ &$0.57 \pm 0.03$ & $0.49 \pm 0.04$ &$0.41 \pm 0.02$ \\
        & 125 & $0.89 \pm 0.03$ & $0.83 \pm 0.04$ & $0.75 \pm 0.03$ & $0.77 \pm 0.03$ &$0.74 \pm 0.02$ & $0.71 \pm 0.03$ &$0.64 \pm 0.02$ \\
        \bottomrule \\
    \end{tabular}
    } 
    \caption{Calibration metrics for VGG-19 on ImageNet-16H.}
    \label{table:imagenet_vgg19_calibration_appendix}
\end{table}

\begin{table}[!ht]
    \centering
    \resizebox{\columnwidth}{!}{%
    \begin{tabular}{llccccccc}
        \toprule
        \multicolumn{2}{c}{} &
        \multicolumn{2}{c}{No Calibration} & 
        \multicolumn{2}{c}{10 Datapoints} &
        \multicolumn{3}{c}{5000 Datapoints}\\
        \cmidrule(lr){3-4} \cmidrule(lr){5-6} \cmidrule(lr){7-9} 
         Metric & Noise Level & Model & Comb. & Model & Comb. & Model & Comb. & Recal. Comb. \\
        \midrule
        \multirow{4}{*}{ECE ($10^{-2}$)} 
        & 80 & $7.39 \pm 0.58$ & $4.10 \pm 0.47$ & $4.60 \pm 0.62$ & $3.04 \pm 0.32$ &$4.52 \pm 0.58$ & $3.07 \pm 0.40$ &$1.97 \pm 0.39$ \\
        & 95 & $9.23 \pm 0.58$ & $5.68 \pm 0.40$ & $5.91 \pm 0.55$ & $4.19 \pm 0.54$ &$5.76 \pm 0.59$ & $4.32 \pm 0.42$ &$2.51 \pm 0.45$ \\
        & 110 & $9.04 \pm 0.68$ & $7.60 \pm 0.46$ & $5.34 \pm 1.15$ & $6.66 \pm 0.71$ &$5.34 \pm 0.37$ & $5.98 \pm 0.50$ &$3.00 \pm 0.43$ \\
        & 125 & $11.98 \pm 0.40$ & $10.95 \pm 0.62$ & $6.78 \pm 1.40$ & $11.33 \pm 0.36$ &$6.54 \pm 0.43$ & $7.98 \pm 0.60$ &$3.37 \pm 0.42$ \\
        \midrule
        \multirow{4}{*}{cwECE ($10^{-2}$)}
        & 80 & $1.33 \pm 0.07$ & $0.67 \pm 0.04$ & $1.30 \pm 0.07$ & $0.63 \pm 0.04$ &$1.30 \pm 0.07$ & $0.57 \pm 0.03$ &$0.57 \pm 0.03$ \\
        & 95 & $1.47 \pm 0.08$ & $0.87 \pm 0.05$ & $1.46 \pm 0.05$ & $0.78 \pm 0.06$ &$1.47 \pm 0.05$ & $0.75 \pm 0.04$ &$0.74 \pm 0.05$ \\
        & 110 & $1.70 \pm 0.08$ & $1.23 \pm 0.06$ & $1.66 \pm 0.03$ & $1.12 \pm 0.07$ &$1.65 \pm 0.04$ & $1.10 \pm 0.06$ &$1.03 \pm 0.06$ \\
        & 125 & $2.31 \pm 0.06$ & $1.92 \pm 0.06$ & $2.19 \pm 0.06$ & $1.70 \pm 0.06$ &$2.19 \pm 0.06$ & $1.68 \pm 0.05$ &$1.59 \pm 0.06$ \\
        \midrule
        \multirow{4}{*}{NLL}
        & 80 & $0.59 \pm 0.03$ & $0.34 \pm 0.02$ & $0.50 \pm 0.02$ & $0.29 \pm 0.02$ &$0.50 \pm 0.03$ & $0.28 \pm 0.02$ &$0.25 \pm 0.02$ \\
        & 95 & $0.65 \pm 0.03$ & $0.43 \pm 0.03$ & $0.56 \pm 0.01$ & $0.37 \pm 0.02$ &$0.56 \pm 0.01$ & $0.36 \pm 0.02$ &$0.33 \pm 0.02$ \\
        & 110 & $0.75 \pm 0.03$ & $0.60 \pm 0.03$ & $0.66 \pm 0.03$ & $0.56 \pm 0.03$ &$0.66 \pm 0.02$ & $0.53 \pm 0.02$ &$0.47 \pm 0.02$ \\
        & 125 & $0.97 \pm 0.02$ & $0.88 \pm 0.03$ & $0.85 \pm 0.02$ & $0.89 \pm 0.02$ &$0.85 \pm 0.01$ & $0.79 \pm 0.02$ &$0.74 \pm 0.02$ \\
        \bottomrule \\
    \end{tabular}
    } 
    \caption{Calibration metrics for GoogLeNet on ImageNet-16H.}
    \label{table:imagenet_googlenet_calibration_appendix}
\end{table}

\FloatBarrier
\newpage
\section{Dataset, Model Training, and Code Details} \label{sec:model_details}

\subsection{CIFAR-10H}
The CIFAR-10H dataset \cite{peterson2019human} consists of the $10,000$ images in the standard CIFAR-10 test set, but each image is labeled by approximately $50$ individual human labelers. There are ten classes in this dataset.

We study four CNN model architectures on CIFAR-10H:
\begin{itemize}
    \item ResNet-110 and ResNet-164 \cite{he2016deep}: Deep residual networks with 110 and 164 layers respectively.
    \item PreResNet-164 \cite{he2016identity}: A deep residual network with identity mappings as skip connections, with 164 layers.
    \item DenseNet-BC \cite{huang2017densely}: A densely connected CNN with $L=190$ layers and a growth-rate of $k=40$, using bottleneck layers.
\end{itemize}

For each model, we use pre-trained weights available at \url{https://github.com/bearpaw/pytorch-classification} (MIT License). These models were trained on the standard CIFAR-10 training split.

\subsection{ImageNet-16H}

The ImageNet-16H dataset \cite{unpublishedwork} consists of noisy images from the ImageNet test set \cite{deng2009imagenet}, distorted by phase noise at each spatial frequency based on four levels of phase noise (80, 95, 110, and 125). Approximately $7200$ images were classified at each noise level (with slight variability per noise level). The number of classes is reduced to $16$ (as compared to $1000$ in the original ImageNet dataset).

We study two model architectures on ImageNet-16H: VGG-19 \cite{simonyan2014very} and GoogLeNet \cite{szegedy2015going}. Our training procedure is detailed as follows. We first load a pre-trained ImageNet model (trained on the original 1000 class ImageNet dataset) from the PyTorch model library \citeappendix{NEURIPS2019_9015}. We remove the final linear layer and replace it with a randomly initialized linear layer with a $16$-dimensional output. We then fine-tune all model weights (using the cross-entropy loss) on noisy images from the ImageNet-16H training set (261,168 images). The models are fine-tuned to all levels of noise simultaneously by randomly assigning a different degree of phase noise (ranging from 0 to 130 degrees) to each training image in a batch.

\subsection{Additional Code Details}

Our experiments are implemented in Python 3.8, and make use of the following libraries:
\begin{itemize}
    \item Scikit-Learn \citeappendix{scikit-learn} (BSD License)
    \item PyTorch \citeappendix{NEURIPS2019_9015} (BSD License)
    \item Pyro \citeappendix{bingham2019pyro} (Apache 2.0 License)
    \item NumPy \citeappendix{harris2020array} (BSD License)
    \item IMax Calibration \cite{patel2020multi} (AGPL-3.0 License)
    \item Ensemble Temperature Scaling \cite{zhang2020mix} (MIT License)
\end{itemize}

\subsection{Compute Resources}
All of our experiments were conducted on a standard desktop computer (AMD Ryzen 5 6-core @ 3.6GHz, 16GB memory).

Other than the fully Bayesian combination, all combination methods studied in this work do not require significant computational resources and can be fit on the order of seconds. The fully Bayesian method (Appendix \ref{sect:fully_bayesian_ts}) is more computationally intensive as it requires the use of MCMC to sample from the posterior distribution over calibration parameters, but can still be fit in approx. 2 minutes with 5000 labeled datapoints. However, we focus on MAP estimation in our main results (which does not require MCMC), and only compare to the fully Bayesian setup as a baseline comparison. In addition, we find the fully Bayesian setup to be less label efficient than the MAP counterpart (see Appendix \ref{sect:addl_learning_curves}).

In terms of model training, our ImageNet-16H models were trained on an internal GPU server with  8x GTX 2080ti GPUs and 2 x Intel Xeon Gold 5218 (16 core) processors. On our hardware, fine-tuning for $50$ epochs requires approximately $6$ hours of training per model. 

\newpage
\section{Conditional Independence Combination as a Special Case of Logistic Regression}
\label{sect:lr_and_ci}

We demonstrate that the conditional independence combination (Equation \eqref{eqn:calibrate_confuse_combo}) can be seen as a special case of logistic regression taking $m(x)$ and $h(x)$ as inputs, but only when the calibration map takes a particular functional form. Calibration maps such as temperature scaling and Dirichlet  calibration \citeappendix{kull2019beyond} satisfy this requirement.

\subsection{Logistic Regression}
In the logistic regression (LR) model, for input $x$ we have features $z \in \R^{2k}$, $z(x) = m(x) \oplus H(x)$, where $H(x)$ is the one-hot version of $h(x)$ and $\oplus$ is the direct sum. A weight matrix $W \in \R^{k \times 2k}$ and a bias $b \in \R^k$ are to be learned. The probabilistic output is given by an element-wise softmax:
\begin{equation}
    x \mapsto \text{SoftMax}(Wz(x) + b) \in \R^k
\end{equation}

We can write $W = [W_m | W_h]$ as a block matrix, where $W_m, W_h \in \R^{k \times k}$ are the model and human weights respectively. In log-space, the LR model is then
\begin{equation}
    \log p(y | m(x), h(x)) =  W_m m(x) + W_h H(x) + b - \log(C)
\end{equation}
where $C$ is a normalizing constant. Since $H(x)$ is one-hot, the term $W_h H(x)$ corresponds to a column in $W_h$, e.g. if $H(x) = [1, 0, \dots, 0]^{\sf T}$, then $W_h H(x)$ is the first column of $W_h$. The above is the full vector of probabilities. To make it clearer, for an index $i$, let $W_m^i$ be the $i$th row of $W_m$ (resp. for $W_h$).

\begin{align}
    p(y = i | m(x), h(x) ) &= W_m^i m(x) + W_h^i H(x) + b_i - \log(C) \\
    &= W_m^i m(x) + (W_h)_{i h(x)} + b_i - \log(C)
\end{align}

\subsection{CI Model}
In the CI model,
\begin{equation}
    p(y | m(x), h(x)) \propto p(y | m(x)) p(h(x) | y) 
\end{equation}
In log-space for a single index $i$:
\begin{align}
    \log p(y=i | m(x), h(x)) &=  \log p(y | m(x)) + \log p(h(x) | y) - \log (C) \\
    &= \log m_i^\theta (x) + \log \varphi_{h(x)y} - \log (C)
\end{align}

\subsection{}
From this, we see that $W_h$ is analogous to the log-confusion matrix of $h$. Similarly, $W_m$ can be thought of as a linear operator mapping the model probabilities to log-calibrated model probabilities.

If we use $\log m(x)$ (pointwise) for the input feature $z(x)$, the LR model is

\begin{align}
    p(y = i | m(x), h(x) ) &= W_m^i \log m(x) + (W_h)_{i h(x)} + b_i - \log(C)
\end{align}

In the special case $W_m = \frac{1}{T} I$, $b_i = 0$, and $W_h = \log \varphi^{\sf T}$, we recover temperature scaling CI. In fact, the equation $W_m \log m(x) + b$ is the same as Dirichlet calibration -- vector scaling / matrix scaling are special cases as well.

\newpage
\section{Derivation of Fully Bayesian Model for TS}
\label{sect:fully_bayesian_ts}

In this section, we derive a fully Bayesian method for combining classifier probabilities with human labels. In summary, we place a Gaussian prior on the log-temperature (for calibration) and independent Dirichlet priors over the columns of the human confusion matrix. The posterior human confusion matrix is available in closed-form (due to conjugacy), and we sample from the posterior distribution over calibration parameters using MCMC. To predict on a new datapoint, we marginalize over the calibration and confusion parameters using the sampled temperatures and closed-form posterior confusion parameters. This marginalization is only approximate due to the required sampling step. 

In more detail, let $\varphi_{*i} \sim \text{Dirichlet}(\alpha_i)$ for $i = 1, 2, \dots, k$ be priors over the columns of the confusion matrix, and let $\log T = \tau \sim \mathcal{N}(\mu_0, \sigma_0^2)$ be a prior over the log-temperature. We use $\varphi$ to denote the confusion matrix with columns $\varphi_{*1} , \dots, \varphi_{*K}$. We assume a fully labeled dataset is available, and of the form $\DD = \{ (h_\ell, m_\ell, y_\ell)\}$. Take the calibration and confusion parameters to be conditionally independent given the data:
\begin{equation}
    p(\tau, \varphi | \DD) = p(\tau | \varphi, \DD) p(\varphi | \DD) = p(\tau | \DD) p(\varphi | \DD)
\end{equation}

The confusion parameters have a conjugate prior, but the calibration parameters do not -- hence, suppose that we have sampled $\{\tau_1, \dots, \tau_{n_s} \}$ from the posterior $p(\tau | \DD)$. To do inference on a new datapoint $(h, m)$, we marginalize over $\varphi$ and $\tau$ for a particular choice of $y$:
\begin{align}
    p(y | h, m, \DD) &= \iint p(y, \tau, \varphi | h, m, \DD) d\varphi d\tau \\
    &= \iint p(y | \tau, \varphi, h, m, \DD) p(\tau | \DD) p(\varphi | \DD) d\varphi d\tau \\
    \intertext{The second line is obtained by conditioning on $\tau, \varphi$ and using the fact that $\tau$ and $\varphi$ are independent given $\DD$. We now use Equation \eqref{eqn:combo_eqn} to re-write the first term, obtaining (up to a constant):
    }
    &\propto \iint  p(h | y, \varphi) p(y | m, \tau) p(\tau | \DD) p(\varphi| \DD)  d\varphi d\tau \\
    \intertext{We now split the integral into its independent components, and use our parametric assumptions to replace $p(h|y, \varphi)$ with $\varphi_{hy}$ and $p(y | m, \tau)$ with $m_y^{(\tau)}$:}
    &= \left[ \int m_y^{(\tau)} p(\tau | \DD) d\tau \right] \left[\int \varphi_{hy} p(\varphi | \DD) d\varphi \right] \\
    \intertext{The second integral is the posterior mean of $\varphi_{hy}$, which is available in closed-form by conjugacy. However, as we do not have a closed-form posterior for $p(\tau | \DD)$, we estimate the first integral using our samples. In all, we obtain}
    &\approx \left[ \frac{1}{n_s} \sum_{j=1}^{n_s} m_y^{(\tau_j)} \right] \cdot \frac{\alpha_{hy}'}{\sum_{\ell=1}^K \alpha_{\ell j}'}
\end{align}

where $\alpha'_{ij}$ is the posterior Dirichlet parameter for entry $(i,j)$ in the confusion matrix $\varphi$. Note that the resulting probabilities will be un-normalized, but normalization is straightforward as we are considering a set of discrete outcomes.

In practice, we use HMC \citeappendix{neal2011mcmc} in the Pyro probabilistic programming language \citeappendix{bingham2019pyro} to sample from the posterior over log-temperatures.

\newpage
\FloatBarrier
\section{Learning Curves}
\label{sect:addl_learning_curves}

In addition to those in Figure \ref{fig:learning_curves}, we provide learning curves that include additional baseline models: logistic regression (LR), the single-parameter confusion matrix method (SP), and the fully Bayesian P+L method (P+L Fully Bayesian). We report only the mean error rate averaged over $10$ random seeds for the sake of visual clarity. All methods (other than LR) are fit using MAP inference. We do not present the maximum likelihood (ML) variants for these methods, as the MAP methods outperform their ML counterparts in our experiments.

While the SP method is label efficient given its low parameter count, it often underfits to the data and converges to an error rate worse than the P+L method. In some cases, the fully Bayesian obtains a lower error rate than the P+L method, but requires more labeled data to be fit, as well as being more computationally intensive. On the CIFAR-10 data, the logistic regression method is label inefficient, and while it outperforms the L+L method, converges to a worse error rate than the P+L method. In contrast, LR is able to outperform the P+L method on the ImageNet-16H datasets, but only when fit several hundred datapoints.

\begin{figure}[hb]
    \centering
    \includegraphics[scale=0.75]{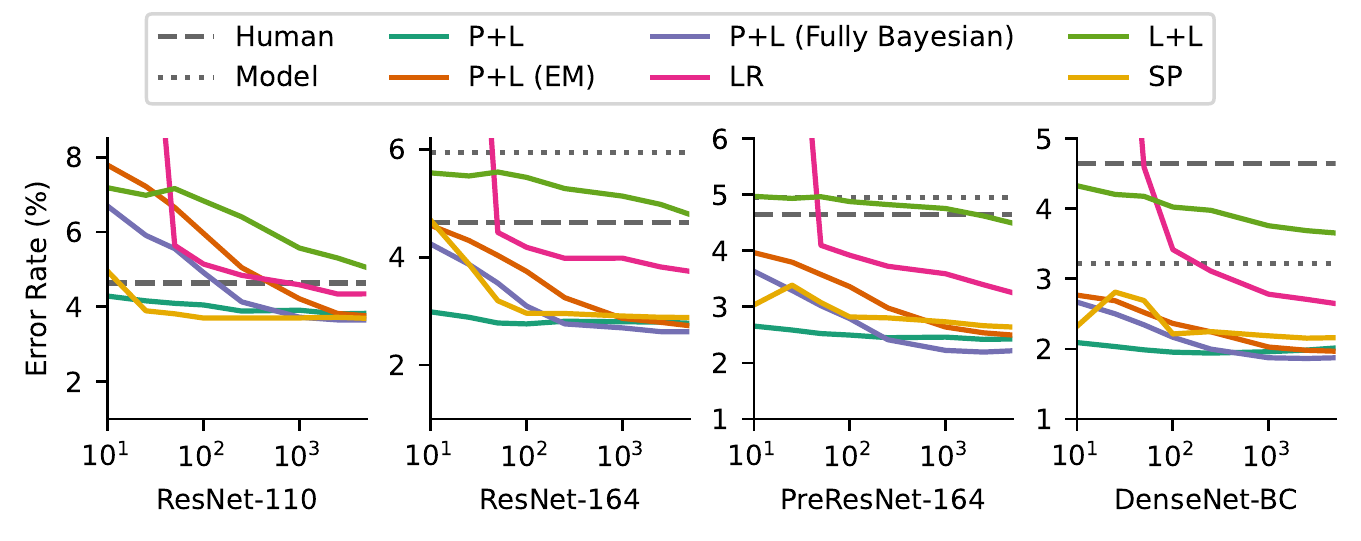}
    \caption{Learning curves for various models on CIFAR-10H.}
    \label{fig:cifar10_learning_curves_appendix}
\end{figure}

\begin{figure}[hb]
    \centering
    \includegraphics[scale=0.75]{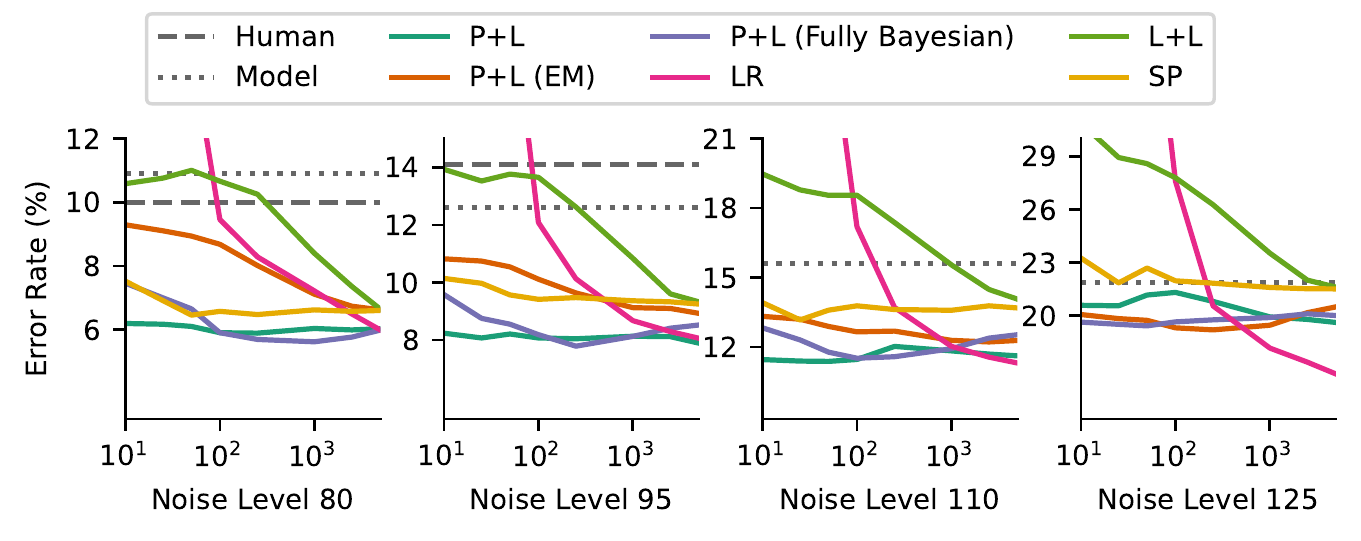}
    \caption{Learning curves for VGG-19 on ImageNet-16H at various noise levels.}
    \label{fig:vgg19_learning_curves_appendix}
\end{figure}

\begin{figure}[hb]
    \centering
    \includegraphics[scale=0.75]{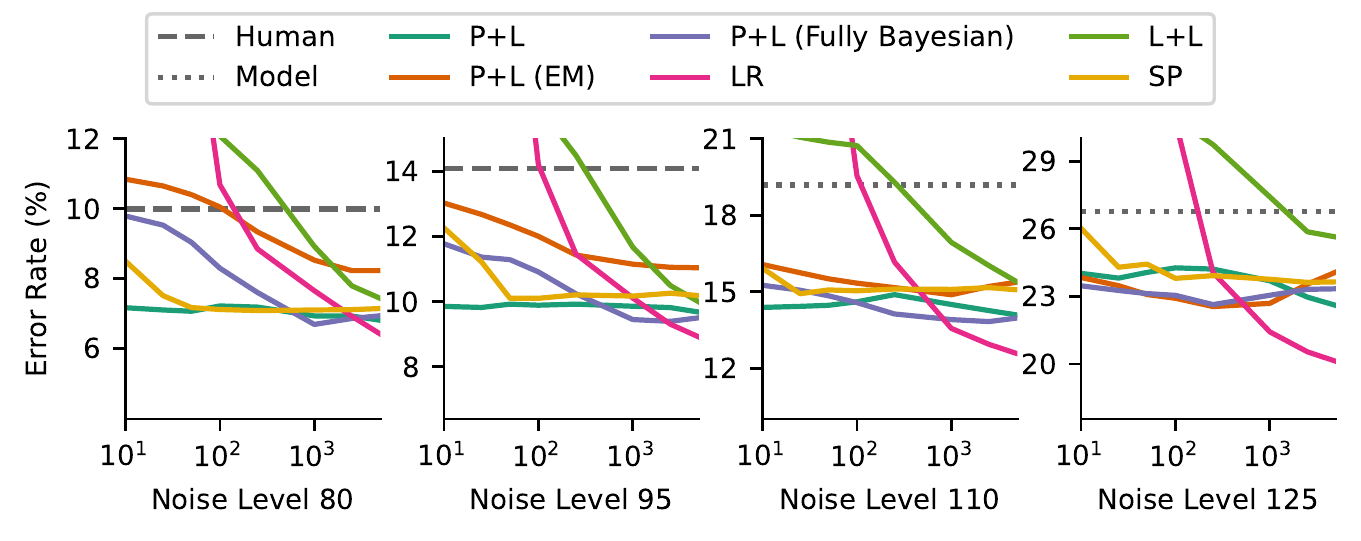}
    \caption{Learning curves for GoogLeNet on ImageNet-16H at various noise levels.}
    \label{fig:googlenet_learning_curves_appendix}
\end{figure}

\FloatBarrier
\newpage
\section{EM Algorithm Details}
\label{sect:em_details}
In this section, we provide a detailed derivation and description of our EM algorithm.

Let $\mathcal{D}_C = \{ (m(x_\ell), h(x_\ell) )  \}_{\ell=1}^n$ be an unlabeled dataset used for fitting combination parameters, consisting of classifier probabilities and human labels but no ground truth labels. Our goal is to infer classifier calibration parameters $\theta$ and the human confusion matrix $\varphi$ from $\DD_C$. We use $m_\ell$ as a shorthand for $m(x_\ell)$ throughout (respectively for $h$).

We can fit this model via EM, where the ground truth is treated as latent. For simplicity, we derive the maximum likelihood variant, and discuss the necessary changes for the MAP variant at the end of this section. In the E-step, $p(y | m_\ell, h_\ell, \varphi, \theta)$ is estimated from Equation \eqref{eqn:calibrate_confuse_combo}. 

For the M-step, we maximize the expected log-likelihood, where we use $\Theta = \{\theta, \varphi \}$ to denote the set of all parameters:
\begin{align*}
    \Theta_{t+1} &= \arg \max_\Theta \sum_i \Exp[y \sim p(y | h, m, \theta_t)]{\log p(y, h_\ell, m_\ell | \Theta)} \\
    &= \arg \max_\Theta \sum_\ell \sum_y p(y | h_\ell, m_\ell, \Theta_t) \log p(y, h_\ell, m_\ell | \Theta) \\
    &= \arg \max_\Theta \bigg[ \sum_\ell \sum_y p(y | h_\ell, m_\ell, \Theta_t) \log p(h_\ell | y, \Theta) \\
    &+ \sum_\ell \sum_y p(y | h_\ell, m_\ell, \Theta_t)  \log  p(y | m_\ell, \Theta) +  C  \bigg]
\end{align*}

where $C$ is a constant not depending on $\Theta$. Assuming further that the calibration and confusion parameters are independent, the M-step becomes two independent optimizations (i.e. one for $\theta$ and one for $\varphi$):
\begin{equation} \label{eqn:m_step_phi}
    \theta_{t+1} = \arg\max_\theta \sum_\ell \sum_y p(y | h_\ell, m_\ell, \Theta_t) \log p(y | m_\ell, \theta)
\end{equation}
\begin{equation} \label{eqn:m_step_psi}
    \varphi_{t+1} = \arg\max_\varphi \sum_\ell \sum_y p(y | h_\ell, m_\ell, \Theta_t )\log p(h_\ell | y, \varphi)
\end{equation}

In Equation \eqref{eqn:m_step_phi}, $\log p(y | m_\ell, \theta)$ depends on the calibration method we choose, and the update for $\theta_{t+1}$ does not have a closed-form update. We use gradient methods to maximize this term.

Equation \eqref{eqn:m_step_psi} is maximum likelihood for the confusion matrix and hence $\varphi_{t+1}$ can be solved for in closed-form. In particular, the value for $\varphi_{t+1}$ at entry $i,j$ is

\begin{equation}
    \varphi_{i,j} = \frac{\sum_{\ell : h_\ell = a} p(y = j | h_\ell, m_\ell, \Theta_t)}{\sum_{\ell} p(y = j | h_\ell, m_\ell, \Theta_t) }
\end{equation}

For the MAP variant of our EM algorithm, our optimizations become

\begin{equation} \label{eqn:m_step_theta_MAP}
    \theta_{t+1} = \arg\max_\theta \sum_\ell \sum_y p(y | h_\ell, m_\ell, \Theta_t) \log p(y | m_\ell, \theta) + \log p(\theta)
\end{equation}
\begin{equation} \label{eqn:m_step_phi_MAP}
    \varphi_{t+1} = \arg\max_\varphi \sum_\ell \sum_y p(y | h_\ell, m_\ell, \Theta_t )\log p(h_\ell | y, \varphi) + \log p(\varphi)
\end{equation}

The first optimization (Equation \eqref{eqn:m_step_theta_MAP}) is still fit using gradient methods. As we choose independent Dirichlet priors for each column of $\varphi$, the closed-form estimate for $\varphi$ becomes

\begin{equation}
    \varphi_{i,j} = \frac{\alpha_{ji} - 1 + \sum_{\ell : h_\ell = a} p(y = j | h_\ell, m_\ell, \Theta_t)}{\gamma + (K-1)\beta - K + \sum_{\ell} p(y = j | h_\ell, m_\ell, \Theta_t) }
\end{equation}

which is analogous to the typical Dirichlet-multinomial posterior.

\bibliographystyleappendix{unsrt}
\bibliographyappendix{appendix_refs}
\end{appendices}

\end{document}